\RequirePackage{etoolbox}
\newtoggle{conference}

\togglefalse{conference} % For the arXive version

\newcommand{\inConference}[1]{\iftoggle{conference}{#1}{}} % For text that should appear in the conference version only
\newcommand{\inArXiv}[1]{\iftoggle{conference}{}{#1}}  % For text that should appear in the arXiv version only
\documentclass[10pt]{article}
\usepackage[paperwidth=7in, paperheight=10in, textwidth=5.25in, textheight=8.2in]{geometry}
\usepackage[ruled,vlined,linesnumbered]{algorithm2e}
\usepackage{subfigure}
\usepackage{amsmath}
\usepackage{amsthm}
\usepackage{amssymb}
\usepackage{natbib}
\usepackage{graphicx}
\usepackage{mathtools}
\usepackage{enumitem}
\usepackage{float}
\usepackage{nicefrac}
\usepackage{forloop}
\usepackage{hyperref}
\usepackage{url}
\usepackage{authblk}
\newtheorem{theorem}{Theorem}
\newtheorem{claim}{Claim}
\newtheorem{observation}{Observation}
\newtheorem{definition}{Definition}
\newtheorem{lemma}{Lemma}

\makeatletter
\newcommand{\newreptheorem}[2]{%
	\newenvironment{rep#1}[1]{%
	\expandafter\renewcommand\csname the#2\endcsname{\ref*{##1}}%
	\expandafter\renewcommand\csname theH#2\endcsname{repeat.##1}%
	\begin{#1}}%
	{\end{#1}%
	\addtocounter{#2}{-1}}}
\makeatother

\newreptheorem{lemma}{theorem}

\DeclarePairedDelimiter{\ceil}{\lceil}{\rceil} % ceiling
\DeclarePairedDelimiter{\evdel}{[}{]} % expectation
\newcommand{\E}{\mathbb{E}\evdel}
\newcommand{\OPT}{\text{OPT}} % optimal solution
\SetKwIF{With}{ElseIf}{Otherwise}{with}{do}{else if}{otherwise}{} % with do otherwise loop
\newcommand{\N}{\mathcal{N}}
\newcommand{\algdt}{\textsc{{RepeatedGreedy}}\xspace}
\newcommand{\algrd}{\textsc{{SampleGreedy}}\xspace}
\newcommand{\nnR}{{\bR_{\geq 0}}}

\def \ee   {\varepsilon}

\def \NN   {{\cal N}}
\def \II   {{\cal I}}
\def \MM   {{\cal M}}

\newcommand{\characteristic}{{\mathbf{1}}}

\newcommand{\defcal}[1]{\expandafter\newcommand\csname c#1\endcsname{{\mathcal{#1}}}}
\newcommand{\defbb}[1]{\expandafter\newcommand\csname b#1\endcsname{{\mathbb{#1}}}}
\newcounter{calBbCounter}
\forLoop{1}{26}{calBbCounter}{
    \edef\letter{\Alph{calBbCounter}}
    \expandafter\defcal\letter
		\expandafter\defbb\letter
}

\title{Greed Is Good: Near-Optimal Submodular\\ Maximization via Greedy Optimization}
\usepackage{times}

\author[1]{Moran Feldman}
\author[2]{Christopher Harshaw}
\author[3]{Amin Karbasi}
\affil[1]{Department of Mathematics and Computer Science, Open University of Israel}
\affil[2]{Department of Computer Science, Yale University}
\affil[3]{Department of Electrical Engineering, Yale University}

\begin{document}

\maketitle

\begin{abstract}
It is known that greedy methods perform well for maximizing \textit{monotone} submodular functions. At the same time, such methods perform poorly in the face of non-monotonicity.  In this paper, we show---arguably, surprisingly---that invoking the classical greedy algorithm $O(\sqrt{k})$-times leads to the (currently) fastest deterministic algorithm, called \algdt, for maximizing a general submodular function subject to $k$-independent system constraints. \algdt achieves $(1 + O(1/\sqrt{k}))k$ approximation using $O(nr\sqrt{k})$ function evaluations (here, $n$ and $r$ denote the size of the ground set and the maximum size of a feasible solution, respectively). We then show that by a careful sampling procedure, we can run the greedy algorithm only \textit{once} and obtain the (currently) fastest randomized algorithm, called \algrd, for maximizing a submodular function subject to $k$-extendible system constraints (a subclass of $k$-independent system constrains). \algrd achieves $(k + 3)$-approximation with only $O(nr/k)$ function evaluations. Finally, we derive an almost matching lower bound, and show that no polynomial time algorithm can have an approximation ratio smaller than $ k + \nicefrac{1}{2} - \varepsilon$. To further support our theoretical results, we compare the performance of \algdt and \algrd with prior art in a concrete application (movie recommendation). We consistently observe that while \algrd achieves practically the same utility as the best baseline, it performs at least two orders of magnitude faster. 

\medskip
\noindent \textbf{Keywords}: Submodular maximization, $k$-systems, $k$-extendible systems, approximation algorithms
\end{abstract}
 
\newpage 

\section{Introduction} \label{sec:intro}

Submodular functions \citep{edmonds1971matroids, fujishige91}, originated in combinatorial optimization
and operations research,  exhibit a natural diminishing returns property common in many
well known objectives: the marginal benefit of any given element decreases as
more and more elements are selected. As a result, submodular 
optimization has found numerous applications in machine learning, including  viral marketing \citep{kempe03}, network monitoring \citep{leskovec07, gomez10}, sensor placement and information gathering \citep{guestrin2005near}, news article recommendation \citep{el2009turning}, nonparametric learning \citep{RG13}, document and corpus summarization \citep{lin2011class,kirchhoff2014submodularity,SSSJ12}, data summarization \citep{mirzasoleiman2013distributed}, crowd teaching \citep{singla2014near}, and MAP inference of determinental point process \citep{GKT12}. The usefulness of submodular optimization in these settings stems from the fact that many such problems can be reduced to the problem of
maximizing a submodular function subject to feasibility constraints, such as cardinality, knapsack, matroid or intersection of matroids constraints. 
%
% \cite{mirzasoleiman2013distributed}, data mining \cite{badanidiyurustreaming}, and natural language processing \cite{lin2011class, wei2014submodular, kirchhoff2014submodularity}. 
%
%Maximization of submodular functions subject to various constraints has attracted a significant attention in recent years (see, e.g.,~[\cite{}]). The main reason for the vast interest in the study of such problems is that many real world and theoretical optimization problems can be cast as special cases of problems of this kind. Thus, maximization of submodular functions subject to general classes of constraints are of particular interest.

 In this paper, we consider the maximization of submodular functions subject to two important classes of constraints known as $k$-system and $k$-extendible system constraints. The class of $k$-system constraints is a very general class of constraints capturing, for example, any constraint which can be represented as the intersection of multiple matroid and matching constraints. The study of the maximization of \emph{monotone} submodular functions subject to a $k$-system constraint goes back to the work of~\cite{Fisher1978}, who showed that the natural greedy algorithm achieves an approximation ratio of $1/(k + 1)$ for this problem ($k$ is a parameter of the $k$-system constraint measuring, intuitively, its complexity). In contrast, results for the maximization of \emph{non-monotone} submodular functions subject to a $k$-system constraint were only obtained much more recently. Specifically, \citet{Gupta2010} showed that, by repeatedly executing the greedy algorithm and an algorithm for unconstrained submodular maximization, one can maximize a non-monotone submodular function subject to a $k$-system constraint up to an approximation ratio of roughly $3k$ using a time complexity of $O(nrk)$\footnote{Here, and throughout the paper, $n$ represents the size of the ground set of the submodular function and $r$ represents the maximum size of any set satisfying the constraint.}. This was recently improved by~\cite{Mirzasoleiman2016}, who showed that the approximation ratio obtained by the above approach is in fact roughly $2k$.

The algorithms we describe in this paper improve over the above mentioned results both in terms of the approximation ratio and in terms of the time complexity. Our first result is a \emph{deterministic} algorithm, called \algdt, which obtains an approximation ratio of $(1 + O(1/\sqrt{k}))k$ for the maximization of a non-monotone submodular function subject to a $k$-system constraint using a time complexity of only $O(nr\sqrt{k})$. \algdt is structurally very similar to the algorithm of~\cite{Gupta2010} and~\cite{Mirzasoleiman2016}. However, thanks to a tighter analysis, it needs to execute the greedy algorithm and the algorithm for unconstrained submodular maximization much less often, which yields its improvement in the time complexity.

Our second result is a \emph{randomized} algorithm, called \algrd, which manages to further push the approximation ratio and time complexity to $(k + 1)^2/k \leq k + 3$ and $O(n + nr/k)$, respectively. However, it does so at a cost. Specifically, \algrd applies only to a  subclass of $k$-system constraints known as $k$-extendible system constraints. We note, however, that the class of $k$-extendible constraints is still general enough to capture, among others, any constraint that can be represented as the intersection of multiple matroid and matching constraints. Interestingly, when the objective function is also monotone or linear the approximation ratio of {\algrd} improves to $k + 1$ and $k$, respectively,\footnote{The improvement for linear objectives requires a minor modification of the algorithm.} which matches the best known approximation ratios for the maximization of such functions subject to a $k$-extendible system constraint \citep{Fisher1978,Jenkyns1976,Mestre2006}. Previously, these ratio were obtained using the greedy algorithm whose time complexity is $O(nr)$. Hence, our algorithm also improves over the state of the art algorithm for maximization of monotone submodular and linear functions subject to a $k$-extendible system constraint in terms of the time complexity.

We complement our algorithmic results with two inapproximability results showing that the approximation ratios obtained by our second algorithm are almost tight. Previously, it was known that no polynomial time algorithm can have an approximation ratio of $k - \varepsilon$ (for any constant $\varepsilon > 0)$ for the problem of maximizing a linear function subject to a $k$-system constraint~\citep{Badanidiyuru14}. We show that this result extends also to $k$-extendible systems, i.e., that no polynomial time algorithm can have an approximation ratio of $k - \epsilon$ for the problem of maximizing a linear function subject to a $k$-extendible system constraint. Moreover, for monotone submodular functions we manage to get a slightly stronger inapproximability result. Namely, we show that no polynomial time algorithm can have an approximation ratio of $1 / (1 - e^{-k}) - \varepsilon \leq k + \nicefrac{1}{2} - \varepsilon$ for the problem of maximizing a monotone submodular function subject to a $k$-extendible system constraint. Note that the gap between the approximation ratio obtained by \algrd (namely,  $(k + 3)$) and the inapproximability result (namely, $(k + \nicefrac{1}{2} - \varepsilon)$) is very small. A short summary of all the results discussed above for non-monotone submodular objectives can be found in Table~\ref{tbl:results}. 

\begin{table}
\begin{center}
\caption{Summary of Results for Non-monotone Submodular Objectives} \label{tbl:results}
\def\arraystretch{1.2}
\begin{tabular}{|l|c|c|}
\hline
& \textbf{\small Approximation Ratio} & \textbf{\small Time Complexity}\\
\hline
\small {\algrd} (randomized) & $(k + 1)^2/k \leq k + 3$ & $O(n + nr/k)$\\
\hline
\small {\algdt} (deterministic) & $k + O(\sqrt{k})$ & $O(nr\sqrt{k})$\\
\hline
\small \cite{Mirzasoleiman2016} & $\approx 2k$ & $O(nrk)$\\
\hline
\small \cite{Gupta2010} & $\approx 3k$ & $O(nrk)$\\
\hline
\small Inapproximability & $(1 - e^{-k})^{-1} - \varepsilon \geq k + \nicefrac{1}{2} - \varepsilon$ & $-$\\
\hline
\end{tabular}
\end{center}
\end{table}

Finally, we compare the performance of \algdt and \algrd against FATNOM, the current state of the art algorithm introduced by \cite{Mirzasoleiman2016}. We test these algorithms on a movie recommendation problem using the MovieLens dataset, which consists of over 20 million ratings of 27,000 movies by 138,000 users. We find that our algorithms provide solution sets of similar quality as FANTOM while running orders of magnitude faster. In fact, we observe that taking the best solution found by several independent executions of \algrd clearly yields the best trade-off between solution quality and computational cost. Moreover, our experimental results indicate that our faster algorithms could be applied to large scale problems previously intractable for other methods.

\paragraph{Organization:}
Section~\ref{sec:related_works} contains a brief summary of additional related work. A formal presentation of our main results and some necessary preliminaries are given in Section~\ref{sec:prelim}. Then, in Sections~\ref{sec:deterministic_alg} and~\ref{sec:randomized_alg} we present and analyze our deterministic and randomized approximation algorithms, respectively. The above mentioned experimental results comparing our algorithms with previously known algorithms can be found in Section~\ref{sec:experimental_results}. Finally, our hardness results appear in Appendix~\ref{sec:hardness}.

\section{Related Works} \label{sec:related_works}

Submodular maximization has been studied with respect to a few special cases of $k$-extendible system constraints. \citet{Lee2010} described a local search algorithm achieving an approximation ratio of $k + \epsilon$ for maximizing a monotone submodular function subject to the intersection of any $k$ matroid constraints, and explained how to use multiple executions of this algorithm to get an approximation ratio of $k + 1 + \frac{1}{k + 1} + \varepsilon$ for this problem even when the objective function is non-monotone. Later, \citet{Feldman2011} showed, via a different analysis, that the same local search algorithm can also be used to get essentially the same results for the maximization of a submodular function subject to a subclass of  $k$-extendible system constraints known as $k$-exchange system constraints (this class of constraints captures the intersection of multiple matching and strongly orderable matroid constraints). 
%
%
%class of constraints known as $k$-exchange system constraints. The last class of constraints is a subclass of the $k$-extendible system constraints which still captures, among others, constraints that can be represented as the intersection of matching and strongly orderable matroid constraints.
For $k \geq 4$, the approximation ratio of \citep{Feldman2011} for the case of a monotone submodular objective was later improved by \citet{Ward2012} to $(k + 3)/2 + \varepsilon$.

An even more special case of $k$-extendible systems are the simple matroid constraints. In their classical work, \citet{Nemhauser1978} showed that the natural greedy algorithm gives an approximation ratio of $1 - 1/e$ for maximizing a monotone submodular function subject to a uniform matroid constraint (also known as cardinality constraint), and \citet{Calinescu2011} later obtained the same approximation ratio for general matroid constraints using the Continuous Greedy algorithm. Moreover, both results are known to be optimal~\citep{Nemhauser1978b}. However, optimization guarantees for non-montone submodular maximization are much less well understood.  After a long series of works~\citep{Vondrak2013,Gharan2011,Feldman2011b,Ene2016}, the current best approximation ratio for maximizing a non-monotone submodular function subject to a matroid constraint is $0.385$ \citep{Buchbinder2016}. In contrast, the state of the art inapproximability result for this problem is $0.478$~\citep{Gharan2011}.  

Recently, there has also been a lot of interest in developing fast algorithms for maximizing submodular functions. \citet{Badanidiyuru14} described algorithms that achieve an approximation ratio of $1 - 1/e - \varepsilon$ for maximizing a monotone submodular function subject to uniform and general matroid constraints using time complexities of $O_{\varepsilon}(n \log k)$ and $O_{\varepsilon}(nr\log^2n)$, respectively ($O_\varepsilon$ suppresses a polynomial dependence on $\varepsilon$).\footnote{\citet{Badanidiyuru14} also describe a fast algorithm for maximizing a monotone submodular function subject to a knapsack constraint. However, the time complexity of this algorithm is exponential in $1/\varepsilon$, and thus, its contribution is mostly theoretical.} For uniform matroid constraints, \citet{Mirzasoleiman15} showed that one can completely get rid of the dependence on $r$, and get an algorithm with the same approximation ratio whose time complexity is only $O_{\varepsilon}(n)$. Independently,~\citet{Buchbinder15} showed that a technique similar to~\citet{Mirzasoleiman15} can be used to get also $(e^{-1} - \varepsilon)$-approximation for maximizing a \emph{non-monotone} submodular function subject to a cardinality constraint using a time complexity of $O_{\varepsilon}(n)$. \citet{Buchbinder15} also described a different $(1 - 1/e - \varepsilon)$-approximation algorithm for maximizing a monotone submodular function subject to a general matroid constraint. The time complexity of this algorithm is $O_{\varepsilon}(r^2 + n\sqrt{r}\log^2 n)$ for general matroids, and it can be improved to $O_{\varepsilon}(r\sqrt{n}\log n + n\log^2 n)$ for generalized partition matroids.

\section{Preliminaries and Main Results} \label{sec:prelim}

In this section we formally describe our main results. However, before we can do that, we first need to present some basic definitions and other preliminaries.

\subsection{Preliminaries}

For a set $A$ and element $e$, we often write the union $A \cup \{e\}$ as $A + e$ for simplicity. Additionally, we say that $f$ is a \emph{real-valued set function} with \emph{ground set} $\N$ if $f$ assigns a real number to each subset of $\N$. We are now ready to introduce the class of submodular functions.

\begin{definition} Let $\N$ be a finite set. A real-valued set function $f: 2^\N \rightarrow \mathbb{R}$ is \emph{submodular} if, for all $X, Y \subseteq \N$,
\begin{equation} \label{eq:submodular_definition}
f(X) + f(Y) \geq f(X \cap Y) + f(X \cup Y) \enspace.
\end{equation} 
Equivalently, for all $A \subseteq B \subseteq \N$ and $e \in \N \setminus B$, 
\begin{equation} \label{eq:diminishing_returns}
f(A + e) - f(A) \geq f(B + e) - f(B) \enspace.
\end{equation}
\end{definition}
While definitions~(\ref{eq:submodular_definition}) and~(\ref{eq:diminishing_returns}) are equivalent, the latter one, which is known as the \emph{diminishing returns property}, tends to be more helpful and intuitive in most situations. Indeed, the fact that submodular functions capture the notion of diminishing returns is one of the key reasons for the usefulness of submodular maximization in combinatorial optimization and machine learning.
%A few examples of submodular functions in combinatorial optimization include weighted vertex cuts in graphs, coverage functions and the facility location objective. Submodular functions appear naturally in machine learning areas such as data summarization, recommender systems, and information gathering.
Due to the importance and usefulness of the diminishing returns property, it is convenient to define, for every $e \in \N$ and $S \subseteq \N$, $\Delta f(e | S) \coloneqq f(S + e) - f(S)$.
In this paper we only consider the maximization of submodular functions which are \emph{non-negative} (i.e., $f(A) \geq 0$ for all $A \subseteq \N$).\footnote{See \citep{Feige2011} for an explanation why submodular functions that can take negative values cannot be maximized even approximately.}

%We are interested in studying various types of submodular functions, which we list here. A submodular function is \emph{nonnegative} if for all $A \subset \N$, $f(A) \geq 0$ and \emph{normalized} if $f(\emptyset) = 0$. A submodular function is \emph{monotone} if for all $A \subset B \subset \N$, $f(A) \leq f(B)$. For example, the cut function of vertex sets in a graph with nonnegative weights is an example of a submodular function that is nonmonotone, nonnegative, and normalized. On the other hand, $g(S) = |S|^{1/2}$ is a submodular function that is monotone, nonnegative, and normalized. In this paper, we are most interested in the general case of submodular functions that are nonmonotone, nonnegative, and normalized.  

As explained above, we consider submodular maximization subject to two kinds of constraints: $k$-system and $k$-extendible system constraints. Both kinds can be cast as special cases of a more general class of constraints known as \emph{independence system} constraints. %In theory and in practice, a submodular function optimization usually involves a set of constraints on the feasbile subsets. We now introduce several types of set systems that appear as constrains in submodular maximization problems.
Formally, an independence system is a pair $(\N, \mathcal{I})$, where $\N$ is a finite set and $\cI$ is a non-empty subset of $2^\N$ having the property that $A \subseteq B \subseteq \cN$ and $B \in \mathcal{I}$ imply together $A \in \mathcal{I}$.

Let us now define some standard terminology related to independence systems. The sets of $\mathcal{I}$ are called the \emph{independent} sets of the independence system. Additionally, an independent set $B$ contained in a subset $X$ of the ground set $\cN$ is a \emph{base} of $X$ if no other independent set $A \subseteq X$ strictly contains $B$. Using this terminology we can now give the formal definition of $k$-systems.

\begin{definition}
An independence system $(\N, \mathcal{I})$ is called a \emph{$k$-system} if for every set $X \subseteq \N$ the sizes of the bases of $X$ differ by at most a factor of $k$. More formally, $|B_1| / |B_2| \leq k$ for every two bases $B_1, B_2$ of $X$.
\end{definition}

An important special case of $k$-systems are the $k$-extendible systems, which were first introduced by \cite{Mestre2006}. We say that an independent set $B$ is an \emph{extension} of an independent set $A$ if $B$ strictly contains $A$.

\begin{definition}
An independence system $(\N, \mathcal{I})$ is {$k$-extendible} if for every independent set $A \in \mathcal{I}$, an extension $B$ of this set and an element $e \notin A$ obeying $A \cup \{e\} \in \mathcal{I}$ there must exist a subset $Y \subseteq B \setminus A$ with $|Y| \leq k$ such that $B \setminus Y \cup \{e\} \in \mathcal{I}$.
\end{definition}

Intuitively, an independence system is $k$-extendible if adding an element $e$ to an independent set $A$ requires the removal of at most $k$ other elements in order to keep the resulting set independent. % The authors of [\cite{Mestre2006}] provide a characterization of $1$-extendible systems as well as several examples of $k$-extendible systems, which we summarize here. An independence system is $1$-extendible if and only if it is a matroid. In fact,
As is shown by~\cite{Mestre2006}, the intersection of $k$ matroids defined on a common ground set is always a $k$-extendible system. The converse is not generally true (except in the case of $k = 1$, since every $1$-system is a matroid). %as is demonstrated by the next example. Suppose that $G=(V,E)$ is an undirected graph and $b: V \rightarrow \mathbb{N}$ is a degree constraint function. A \emph{b-matching} is a set $M$ of edges such that for every vertex $v \in V$ the number of edges in $M$ incident to $v$ is at most $b(v)$. If $b(v) = 1$ for all vertices $v$, then $b$-matching reduces to the usual notion of matching in a graph. For any degree constraint function $b$, the independence system in which a set is independent if and only if it is a $b$-matching is 2-extendible, however, this independence system cannot, in general, be represented as the intersection of a constant number of matroids.% More examples are provided in [\cite{Mestre2006}].

%The author of [\cite{Mestre2015}] show that the intersection of a $k_1$-system and a $k_2$-system is a $(k_1 + k_2)$-system. Moreover, the intersection of a $k_1$-exchange system and a $k_2$-exchange system is a $(k_1 + k_2)$-exchange system. These intersection results extend the obvious fact that the intersection of a $k_1$-matroid intersection and a $k_2$-matroid intersection is a $k_1 + k_2$-matroid intersection to more general classes of independence systems. Clearly, $k$-systems and $k$-extendible systems cover a broad range of constraints that are of theoretical and practical interest.

%In this paper, we describe algorithms for and analyze the hardness of the following problem. Suppose that $\N$ is a finite set, $f:2^\N \rightarrow \mathbb{R}$ is a nonnegative normalized submodular function, and $(\N, \mathcal{I})$ is $k$-system which is possibly $k$-extendible. The goal is to find a solution to the following optimization problem. 
%\begin{equation} \label{eq:opt_problem}
%\max \limits_{S \in \mathcal{I}} f(S)
%\end{equation}
%As one might expect, (\ref{eq:opt_problem}) generalizes a variety of NP-Hard problems and so is NP-Hard. As such, our goal is not to solve (\ref{eq:opt_problem}) exactly, but to develop efficient algorithms that produce constant factor approximations. More precisely, for some $\alpha > 1$ independent of $|\N|$, these algorithms will produce a set $S \in \mathcal{I}$ such that $f(S) \geq 1 / \alpha f(\OPT)$ where $\OPT$ is an optimal solution to (\ref{eq:opt_problem}).

\subsection{Main Contributions} \label{subsec:main_contributions}
Our main contributions in this paper are two efficient algorithms for submodular maximization: one deterministic and one randomized. The following theorem formally describes the properties of our randomized algorithm. Recall that $n$ is the size of the ground set and $r$ is the size of the maximal feasible set.% , i.e., the size of the maximal independent set of the independence system defining the constraint.

\begin{theorem} \label{thm:randomized_alg}
Let $f\colon 2^\cN \to \nnR$ be a non-negative submodular function, and let $(\N, \mathcal{I})$ be a $k$-extendible system. Then, there exists an $O(n + nr / k)$ time algorithm for maximizing $f$ over $(\N, \cI)$ whose approximation ratio is at least $\frac{(k+1)^2}{k}$. Moreover, if the function $f$ is also monotone or linear, then the approximation ratio of the algorithm improves to $k + 1$ and $k$, respectively (in the case of a linear objective, the improvement requires a minor modification of the algorithm).
\end{theorem}

Our deterministic algorithm uses an algorithm for unconstrained submodular maximization as a subroutine, and its properties depend on the exact properties of this subroutine. Let us denote by $\alpha$ the approximation ratio of this subroutine and by $T(n)$ its time complexity given a ground set of size $n$. Then, as long as the subroutine is deterministic, our algorithm has the following properties.

\begin{theorem} \label{thm:deterministic_alg}
Let $f\colon 2^\cN \to \nnR$ be a non-negative submodular function, and let $(\N, \mathcal{I})$ be a $k$-system. Then, there exists a deterministic $O((nr + T(r))\sqrt{k})$ time algorithm for maximizing $f$ over $(\N, \cI)$ whose approximation ratio is at least $k + \left(1 + \frac{\alpha}{2}\right) \sqrt{k} + 2 + \frac{\alpha}{2} + O(1/\sqrt{k})$.
\end{theorem}

\citet{Buchbinder2015} provide a deterministic linear-time algorithm for unconstrained submodular maximization having an approximation ratio of $3$. Using this deterministic algorithm as the subroutine, the approximation ratio of the algorithm guaranteed by Theorem~\ref{thm:deterministic_alg} becomes $k + \frac{5}{2} \sqrt{k} + \frac{7}{2} + O(1/\sqrt{k})$ and its time complexity becomes $O(nr\sqrt{k})$. We note that \citet{Buchbinder2016b} recently came up with a deterministic algorithm for unconstrained submodular maximization having an optimal approximation ratio of $2$. Using this algorithm instead of the deterministic algorithm of~\citet{Buchbinder2015} could marginally improve the approximation ratio guaranteed by Theorem~\ref{thm:deterministic_alg}. However, the algorithm of~\citet{Buchbinder2016} has a quadratic time complexity, and thus, it is less practical. It is also worth noting that \citet{Buchbinder2015} also describe a randomized linear-time algorithm for unconstrained submodular maximization which again achieves the optimal approximation ratio of $2$. It turns out that one can get a randomized algorithm whose approximation ratio is $k + 2\sqrt{k} + 3 + O(1/\sqrt{k})$ by plugging the randomized algorithm of \citep{Buchbinder2015} as a subroutine into the algorithm guaranteed by Theorem~\ref{thm:deterministic_alg}. However, the obtained randomized algorithm is only useful for the rare case of a $k$-system constraint which is not also a $k$-extendible system constraint since Theorem~\ref{thm:randomized_alg} already provides a randomized algorithm with a better guarantee for constraints of the later kind.

We end this section with our inapproximability results for maximizing linear and submodular functions over $k$-extendible systems. Recall that these inapproximability results nearly match the approximation ratio of the algorithm given by Theorem~\ref{thm:randomized_alg}.

\begin{theorem} \label{thm:linear_hardness}
There is no polynomial time algorithm for maximizing a linear function over a $k$-extendible system that achieves an approximation ratio of $k - \varepsilon$ for any constant $\varepsilon > 0$.
\end{theorem}

\begin{theorem} \label{thm:submodular_hardness}
There is no polynomial time algorithm for maximizing a non-negative monotone submodular function over a $k$-extendible system that achieves an approximation ratio of $(1 - e^{-1/k})^{-1} - \varepsilon$ for any constant $\varepsilon > 0$.
\end{theorem}

We note that the inapproximability results given by the last two theorems apply also to a fractional relaxation of the corresponding problems known as the multilinear relaxation. This observation has two implications. The first of these implications is that even if we were willing to settle for a fractional solution, still we could not get a better than $(1 / (1 - e^{-k}))$-approximation for maximizing a monotone submodular function subject to a $k$-extendible system constraint. Interestingly, this approximation ratio can in fact be reached in the fractional case even for $k$-system constraints using the well known (computationally heavy) Continuous Greedy algorithm of \citet{Calinescu2011}. Thus, we get the second implication of the above observation, which is that further improving our inapproximability results requires the use of new techniques since the current technique cannot lead to different results for the fractional and non-fractional problems.

\section{Repeated Greedy: An Efficient Deterministic Algorithm}
 \label{sec:deterministic_alg}

In this section, we present and analyze the deterministic algorithm for maximizing a submodular function $f$ subject to a $k$-system constraint whose existence is guaranteed by Theorem~\ref{thm:deterministic_alg}. Our algorithm works in iterations, and in each iteration it makes three operations: executing the greedy algorithm to produce a feasible set, executing a deterministic unconstrained submodular maximization algorithm on the output set of the greedy algorithm to produce a second feasible set, and removing the elements of the set produced by the greedy algorithm from the ground set. After the algorithm makes $\ell$ iterations of this kind, where $\ell$ is a parameter to be determined later, the algorithm terminates and outputs the best set among all the feasible sets encountered during its iterations. A more formal statement of this algorithm is given as Algorithm~\ref{alg:repeated_greedy}.

% re format the algorithms
%\RestyleAlgo{boxed}
%\SetAlgoVlined
%\LinesNumbered
%{
%\RestyleAlgo{boxed,vlined,linesnumbered}
\begin{algorithm}
\caption{Repeated Greedy($\N, f, \mathcal{I}, \ell$)}\label{alg:repeated_greedy}
Let $\N_1 \gets \N$.\\
\For{$i = 1$ \KwTo $\ell$}
{
	Let $S_i$ be the output of the greedy algorithm given $\N_i$ as the ground set, $f$ as the objective and $\mathcal{I}$ as the constraint.\\
	Let $S'_i$ be the output of a deterministic algorithm for unconstrained submodular maximization given $S_i$ as the ground set and $f$ as the objective.\\
	Let $\N_{i + 1} \gets \N_i \setminus S_i$.
}
\Return the set $T$ maximizing $f$ among the sets $\{S_i, S'_i\}_{i = 1}^\ell$.
\end{algorithm}
%}

%\begin{algorithm}[H] \label{alg:repeated_greedy}
%\caption{Repeated Greedy}
%\DontPrintSemicolon
%\SetKwInOut{Input}{Input}\SetKwInOut{Output}{Output}
%\Input{ground set $\N$, submodular function $f$, $k$-system $(\N,\mathcal{I})$, iterations $r$}
%\Output{solution set $T$}
%\BlankLine
%Let $\N_1 \leftarrow \N$\\
%\For{$i=1$ \KwTo $r$}{
%	Let $S_i$ be the output of the greedy algorithm given $\N_i$ as the ground set, $f$ as the objective, and $\mathcal{I}$ as the constraint.\\
%	Let $S_i'$ be the output of an optimal unconstrained submodular maximization algorithm given $S_i$ as the ground set and $f$ as the objective.\\
%	$\N_{i+1} \leftarrow \N_i \setminus S_i$\\
%}
%Let $T$ be the set maximizing $f$ among the sets $\{S_i, S_i'\}_{i=1}^r$.\\
%\Return{$T$}\\
%\end{algorithm}

\begin{observation}
The set $T$ returned by Algorithm~\ref{alg:repeated_greedy} is independent.
\end{observation}
\begin{proof}
For every $1 \leq i \leq \ell$, the set $S_i$ is independent because the greedy algorithm returns an independent set. Moreover, the set $S'_i$ is also independent since $(\N,\mathcal{I})$ is an independence system and the algorithm for unconstrained maximization must return a subset of its independent ground set $S_i$. The observation now follows since $T$ is chosen as one of these sets.
\end{proof}\vspace{0pt}

We now begin the analysis of the approximation ratio of Algorithm~\ref{alg:repeated_greedy}. Let $\OPT$ be an independent set of $(\cN, \cI)$ maximizing $f$, and let $\alpha$ be the approximation ratio of the unconstrained submodular maximization algorithm used by Algorithm~\ref{alg:repeated_greedy}. The analysis is based on three lemmata. The first of these lemmata states properties of the sets $S_i$ and $S'_i$ which follow immediately from the definition of $\alpha$ and known results about the greedy algorithm.\inConference{ Due to space constraints, the proofs of all three lemmata have been deferred to Appendix~\ref{app:missing_proofs_deterministic}.}

\newcommand{\knownApproxResultsLemma}{
For every $1 \leq i \leq \ell$, $f(S_i) \geq \frac{1}{k+1}f(S_i \cup (\OPT \cap \N_i))$ and $f(S_i') \geq f(S_i \cap \OPT)/\alpha$.
}
\begin{lemma} \label{lem:known_approx_results}
\knownApproxResultsLemma
\end{lemma}
\newcommand{\knownApproxResultsProof}{
\begin{proof}
The set $S_i$ is the output of the greedy algorithm when executed on the $k$-system obtained by restricting $(\cN, \cI)$ to the ground set $\cN_i$. Note that $\OPT \cap \N_i$ is an independent set of this $k$-system. Thus, the inequality $f(S_i) \geq \frac{1}{k+1} f(S_i \cup (\OPT \cap \N_i))$ is a direct application of Lemma~3.2 of [\cite{Gupta2010}] which states that the sets $S$ obtained by running greedy with a $k$-system constraint must obey $f(S) \geq \frac{1}{k+1} f(S \cup C)$ for all independent sets $C$ of the $k$-system.

Let us now explain why the second inequality of the lemma holds. Suppose that $\OPT_i$ is the subset of $S_i$ maximizing $f$. Then,
\[f(S_i') \geq f(\OPT_i) / \alpha \geq f(S_i \cap \OPT) / \alpha \enspace, \]
where the first inequality follows since $\alpha$ is the approximation ratio of the algorithm used for unconstrained submodular maximization, and the second inequality follows from the definition of $\OPT_i$.
\end{proof}\vspace{0pt}
}\inArXiv{\knownApproxResultsProof}

The next lemma shows that the average value of the union between a set from $\{S_i\}_{i = 1}^\ell$ and $\OPT$ must be quite large. Intuitively this follows from the fact that these sets are disjoint, and thus, every ``bad'' element which decreases the value of $\OPT$ can appear only in one of them.

\newcommand{\sumGreedyResultsLemma}{
$\sum \limits_{i=1}^\ell f(S_i \cup \OPT) \geq (\ell-1) f(\OPT)$.
}
\begin{lemma} \label{lem:sum_greedy_results}
\sumGreedyResultsLemma
\end{lemma}
\newcommand{\sumGreedyResultsProof}{
\begin{proof}
The proof is based on the following known result.
\begin{claim}[Lemma~2.2 of \cite{Buchbinder2014}] \label{lem:buchbinder}
Let $g:2^\N \rightarrow \nnR$ be non-negative and submodular, and let $S$ a random subset of $\N$ where each element appears with probability at most $p$ (not necessarily independently). Then, $\E{g(S)} \geq (1 - p) g( \varnothing)$.
\end{claim}

Using this claim we can now prove the lemma as follows. Let $S$ be a random set which is equal to every one of the sets $\{S_i \}_{i=1}^\ell$ with probability $\frac{1}{\ell}$. Since these sets are disjoint, every element of $\N$ belongs to $S$ with probability at most $p = \frac{1}{r}$. Additionally, let us define $g:2^\N \rightarrow \nnR$ as $g(T) = f(T \cup \OPT)$ for every $T \subseteq \cN$. One can observe that $g$ is non-negative and submodular, and thus, by Claim~\ref{lem:buchbinder},
\[
	\frac{1}{\ell} \sum \limits_{i=1}^\ell f( S_i \cup \OPT)
	=
	\E{f(S \cup \OPT)}
	=
	\E{g(S)}
	\geq
	(1 - p) g(\varnothing)
	=
	\left( 1 - \frac{1}{\ell} \right) f(\OPT)
	\enspace.
\]
The lemma now follows by multiplying both sides of the last inequality by $\ell$.
\end{proof}\vspace{0pt}
}\inArXiv{\sumGreedyResultsProof}

The final lemma we need is the following basic fact about submodular functions.

\newcommand{\submodularFactLemma}{
Suppose $f$ is a non-negative submodular function over ground set $\N$. For every three sets $A, B, C \subseteq \N$, $f(A \cup (B \cap C)) + f(B \setminus C) \geq f(A \cup B)$.
}
\begin{lemma}\label{lem:submodular_fact1}
\submodularFactLemma
\end{lemma}
\newcommand{\submodularFactProof}{
\begin{proof}
Observe that
\begin{align*}
	f(A \cup (B \cap C)) + f(B \setminus C)
	\geq{} &
	f(A \cup (B \cap C) \cup (B \setminus C)) + f((A \cup (B \cap C)) \cap (B \setminus C))\\
	\geq{} &
	f(A \cup (B \cap C) \cup (B \setminus C))
	=
	f(A \cup B)
	\enspace,
\end{align*}
where the first inequality follows from the submodularity of $f$, and the second inequality follows from its non-negativity.
\end{proof}\vspace{0pt}
}\inArXiv{\submodularFactProof}

Having the above three lemmata, we are now ready to prove Theorem~\ref{thm:deterministic_alg}.

%\begin{theorem}
%Let $\alpha$ be the approximation ratio for the the unconstrained submodular optimization algorithm used in Algorithm~\ref{alg:repeated_greedy}. If $r=\ceil{\sqrt{k}}$ then Algorithm~\ref{alg:repeated_greedy} has an approximation ratio of at most ${k + \left(1 + \frac{\alpha}{2}\right) \sqrt{k} + 2 + \frac{\alpha}{2}$
%\end{theorem}
\newcommand{\calculations}{
\begin{align*}
\frac{k+ \frac{\alpha}{2}\ell + 1 - \frac{\alpha}{2}}{1-\frac{1}{\ell}}
	&\leq \frac{k+ \frac{\alpha}{2}(\sqrt{k} + 1) + 1 - \frac{\alpha}{2}}{1-\frac{1}{\sqrt{k}}}
	= \frac{k + \frac{\alpha}{2}\sqrt{k} + 1}{1-\frac{1}{\sqrt{k}}}\\
	&\leq k + \left( 1+ \frac{\alpha}{2} \right) \sqrt{k} + 2 + \frac{\alpha}{2} + \frac{4 + \alpha}{\sqrt{k}}
	\enspace,
\end{align*}
where the last inequality holds because for $k \geq 4$ it holds that
\begin{align*}
	\left[k + \left(1+\frac{\alpha}{2}\right)\sqrt{k} + 2 + \frac{\alpha}{2} + \frac{4 + \alpha}{\sqrt{k}}\right] \left( 1 - \frac{1}{\sqrt{k}} \right)
	={} &
	k + \frac{\alpha}{2}\sqrt{k} + 1 + \frac{4 + \alpha}{2\sqrt{k}} - \frac{4 + \alpha}{k}\\
	\geq{} &
	k + \frac{\alpha}{2}\sqrt{k} + 1
	\enspace. \inArXiv{\qedhere}
\end{align*}
}
\begin{proof}[Proof of Theorem~\ref{thm:deterministic_alg}]
Observe that, for every $1 \leq i \leq \ell$, we have
\begin{equation} \label{eq:set_equalities}
\OPT \setminus \N_i = \OPT \cap ( \N \setminus \N_i) = \OPT \cap \left( \cup_{j=1}^{i-1} S_i \right) = \cup_{j=1}^{i-1} \left( \OPT \cap S_j \right)
\end{equation}
where the first equality holds because $\OPT \subseteq \N$ and the second equality follows from the removal of $S_i$ from the ground set in each iteration of Algorithm~\ref{alg:repeated_greedy}.
%Observe that for every $i=1 \dots r$, we have that
%\begin{equation} \label{eq:leftover_elements}
%\N \setminus \N_i = \bigcup \limits_{j=1}^{i-1} S_j
%\end{equation}
Using the previous lemmata and this observation, we get
{\allowdisplaybreaks
\begin{align*}
(\ell-1) f(\OPT) & \leq \sum \limits_{i=1}^\ell f(S_i \cup \OPT) &\text{(Lemma~\ref{lem:sum_greedy_results})} \\
	&\leq \sum \limits_{i=1}^\ell f(S_i \cup (\OPT \cap \N_i)) + \sum \limits_{i=1}^\ell f(\OPT \setminus \N_i) &\text{(Lemma~\ref{lem:submodular_fact1})} \\
%	&= \sum \limits_{i=1}^r f(S_i \cup (\OPT \cap \N_i)) + \sum \limits_{i=1}^r f(\OPT \cap (\N \setminus \N_i)) &\text{because $\N_i, \OPT \subset \N$}\\
%	&= \sum \limits_{i=1}^r f(S_i \cup (\OPT \cap \N_i)) + \sum \limits_{i=1}^r f \left( \OPT \cap (\cup_{j=1}^{i-1} S_j) \right) &\text{by \ref{eq:leftover_elements}} \\
	&= \sum \limits_{i=1}^\ell f(S_i \cup (\OPT \cap \N_i)) + \sum \limits_{i=1}^\ell f \left(\cup_{j=1}^{i-1}(\OPT \cap S_j) \right) &\text{(Equality~\eqref{eq:set_equalities})} \\
	&\leq \sum \limits_{i=1}^\ell f(S_i \cup (\OPT \cap \N_i)) + \sum \limits_{i=1}^\ell \sum \limits_{j=1}^{i-1} f(\OPT \cap S_j) &\text{(submodularity)} \\
	&\leq (k+1) \sum \limits_{i=1}^\ell f(S_i) + \alpha \sum \limits_{i=1}^\ell \sum \limits_{j=1}^{i-1} f(S_j') &\text{(Lemma~\ref{lem:known_approx_results})} \\
	&\leq (k+1) \sum \limits_{i=1}^\ell f(T) + \alpha \sum \limits_{i=1}^\ell \sum \limits_{j=1}^{i-1} f(T) &\text{($T$'s definition)} \\
	&= \left[ (k+1)\ell + \alpha \ell (\ell-1) /2 \right] f(T)
	\enspace.
\end{align*}
}
Dividing the last inequality by $(k+1)\ell + \alpha \ell (\ell-1) /2$, we get
\begin{equation} \label{eq:approximation_ratio}
f(T) \geq \frac{\ell-1}{(k+1)\ell + \frac{\alpha}{2}\ell(\ell-1)} f(\OPT) = \frac{1-\frac{1}{\ell}}{ k+ \frac{\alpha}{2}\ell + 1 - \frac{\alpha}{2} } f(\OPT) \enspace.
\end{equation}
%\begin{align*}
%f(T) &\geq \frac{r-1}{(k+1)r + \alpha \sum \limits_{i=1}^r (i-1)} f(\OPT) \\
%	&= \frac{r-1}{(k+1)r + \frac{\alpha}{2}r(r-1)} f(\OPT) \\
%	&= \frac{r-1}{r(k+ \frac{\alpha}{2}r + (1 - \frac{\alpha}{2}))} f(\OPT)
%\end{align*}

The last inequality shows that the approximation ratio of Algorithm~\ref{alg:repeated_greedy} is at most $\frac{ k+ \frac{\alpha}{2}\ell + 1 - \frac{\alpha}{2} }{1-\frac{1}{\ell}}$. To prove the theorem it remains to show that, for an appropriate choice of $\ell$, this ratio is at most $k + \left( 1+ \frac{\alpha}{2} \right) \sqrt{k} + 2 + \frac{\alpha}{2} + O(1/\sqrt{k})$. It turns out that the right value of $\ell$ for us is $\ceil{\sqrt{k}}$. \inConference{The theorem follows by plugging this value into Inequality~\eqref{eq:approximation_ratio}. See Appendix~\ref{app:missing_proofs_deterministic} for the exact calculations.}\inArXiv{Plugging this value into Inequality~\eqref{eq:approximation_ratio} we get that the approximation ratio of Algorithm~\ref{alg:repeated_greedy} is at most}
\inArXiv{\calculations}
\end{proof}\vspace{0pt}

\section{Sample Greedy: An Efficient Randomized Algorithm} \label{sec:randomized_alg}
In this section, we present and analyze a randomized algorithm for maximizing a submodular function $f$ subject to a $k$-extendible system constraint. Our algorithm is very simple: it first samples elements from $\N$, and then runs the greedy algorithm on the sampled set. This algorithm is outlined as Algorithm~\ref{alg:sample_greedy}.

\SetKwIF{With}{OtherwiseWith}{Otherwise}{with}{do}{otherwise with}{otherwise}{}
\noindent \begin{minipage}[c]{0.45\textwidth}
  %\vspace{0pt}  
  % Algorithm 1
  \begin{algorithm}[H]
	\caption{\newline\hspace*{1cm}Sample Greedy($\cN, f, \mathcal{I}, k$)} \label{alg:sample_greedy}
	\DontPrintSemicolon
%	\SetKwInOut{Input}{Input}\SetKwInOut{Output}{Output}
%	\Input{ground set $\N$, submodular function $f$, $k$-system $(\N, \mathcal{I})$}
%	\Output{solution set $S$}
	%\BlankLine
	Let $\N' \leftarrow \varnothing$ and $S \leftarrow \varnothing$.\\
	\For{$u \in \N$}{
		\With{probability $(k+1)^{-1}$\label{line:sample}}{Add $u$ to $\N'$.}
	}
	\While{there exists $\ u \in \N'$ such that $S + u \in \mathcal{I}$ and $\Delta f(u | S) > 0$}{
		Let $u \in \N'$ be the element of this kind maximizing $\Delta f(u | S)$.\\
		Add $u$ to $S$.\\
	}
	\Return{$S$}.\\
	\end{algorithm}
\end{minipage}
\hspace{0.01\textwidth}
\begin{minipage}[c]{0.53\textwidth}
  %\vspace{0pt}
  % Algorithm 2
  \begin{algorithm}[H]
	\caption{\newline\hspace*{1cm}Equivalent Algorithm($\cN, f, \mathcal{I}, k$)} \label{alg:equiv_sample_greedy}
	\DontPrintSemicolon
	Let $\cN' \gets \cN$, $S \gets \varnothing$ and $O \gets OPT$.\\
	\While{there exists an element $u \in \cN'$ such that $S + u \in \cI$ and $\Delta f(u | S) > 0$}
	{
		Let $u \in \cN'$ be the element of this kind maximizing $\Delta f(u | S)$, and let $S_u \gets S$.\\
		\With{probability $(k + 1)^{-1}$\label{line:sample_equiv}}
		{
			Add $u$ to $S$ and $O$.\\
			Let $O_u \subseteq O \setminus S$ be the smallest set such that $O \setminus O_u \in \cI$.\\
		}
		\Otherwise
		{
			\lIf{$u \in O$}{Let $O_u \gets \{u\}$.}
			\lElse{Let $O_u \gets \varnothing$.}
		}
		Remove the elements of $O_u$ from $O$.\\
		Remove $u$ from $\cN'$.\\
	}
	\Return{$S$.}
\end{algorithm}
\end{minipage}

%\begin{algorithm}[h]
%\caption{\label{alg:sample_greedy} Sample Greedy}
%\DontPrintSemicolon
%\SetKwInOut{Input}{Input}\SetKwInOut{Output}{Output}
%\Input{ground set $\N$, submodular function $f$, $k$-system $(\N, \mathcal{I})$}
%\Output{solution set $S$}
%\BlankLine
%Let $\tilde{\N} \leftarrow \emptyset$ and $S \leftarrow \emptyset$\\
%\For{$u \in \N$}{
%	\With{probability $(k+1)^{-1}$}{add $u$ to $\tilde{\N}$}
%}
%\While{$\exists \ u \in \N'$ such that $S \cup \{u\} \in \mathcal{I}$ and $\Delta f(u | S) > 0$}{
%	Let $u \in \N'$ be the element of this kind maximizing $\Delta f(u | S) > 0$\\
%	Add $u$ to $S$
%}
%\Return{$S$}\\
%\end{algorithm}

To better analyze Algorithm~\ref{alg:sample_greedy}, we introduce an auxiliary algorithm given as Algorithm~\ref{alg:equiv_sample_greedy}. It is not difficult to see that both algorithms have identical output distributions. The sampling of elements in Algorithm~\ref{alg:sample_greedy} is independent of the greedy maximization, so interchanging these two steps does not affect the output distribution. Moreover, the variables $S_u$, $O$ and $O_u$ in Algorithm~\ref{alg:equiv_sample_greedy} do not affect the output $S$ (and in fact, appear only for analysis purposes). Thus, the two algorithms are equivalent, and any approximation guarantee we can prove for Algorithm~\ref{alg:equiv_sample_greedy} immediately carries over to Algorithm~\ref{alg:sample_greedy}.

%\begin{algorithm}[h] 
%\caption{\label{alg:equiv_sample_greedy} Equivalent Randomized Algorithm}
%\DontPrintSemicolon
%\SetKwInOut{Input}{Input}\SetKwInOut{Output}{Output}
%\Input{ground set $\N$, submodular function $f$, $k$-system $(\N, \mathcal{I})$}
%\Output{solution set $S$}
%\BlankLine
%Let $\N' \leftarrow \N$, $S \leftarrow \emptyset$ and $O \leftarrow \OPT$\\
%\While{$\exists \ u \in \N'$ such that $S \cup \{u\} \in \mathcal{I}$ and $\Delta f(u | S) > 0$}{
%	Let $u \in \N'$ be the element of this kind maximizing $\Delta f(u | S) > 0$\\
%	$S_u \leftarrow S$ \\
%	\With{probability $(k+1)^{-1}$}{
%		Add $u$ to $S$ and $O$\\
%		\If{$O \notin \mathcal{I}$}{
%			Let $O_u \subset O \setminus S$ be a set of up to $k$ elements such that $O \setminus O_u \in \mathcal{I}$ \\
%			Remove the elements of $O_u$ from $O$
%		}
%	}
%	\Otherwise{
%		\If{$u \in O$}{Remove $u$ from $O$, and let $O_u \leftarrow \{u\}$}
%	}
%	Remove $u$ from $\N'$
%}
%\Return{$S$}
%\end{algorithm}

Next, we are going to analyze Algorithm~\ref{alg:equiv_sample_greedy}. However, before doing it, let us intuitively explain the roles of the sets $S$, $S_u$, $O$ and $O_u$ in this algorithm. We say that Algorithm~\ref{alg:equiv_sample_greedy} considers an element $u$ in some iteration if $u$ is the element chosen as maximizing $\Delta f(u|S)$ at the beginning of this iteration. Note that an element is considered at most once, and perhaps not at all. As in Algorithm~\ref{alg:sample_greedy}, $S$ is the current solution.  Likewise, $S_u$ is the current solution $S$ at the beginning of the iteration in which Algorithm~\ref{alg:equiv_sample_greedy} considers $u$. The set $O$ maintained in Algorithm~\ref{alg:equiv_sample_greedy} is an independent set which starts as $\OPT$ and changes over time, while preserving three properties: 
\begin{enumerate}[label=\textbf{P\arabic*}]
	\item $O$ is an independent set. \label{itm:O_ind}
    \item Every element of $S$ is an element of $O$. \label{itm:S_sub_O}
    \item Every element of $O \setminus S$ is an element not yet considered by Algorithm~\ref{alg:equiv_sample_greedy}. \label{itm:not_considered}
  \end{enumerate}
Because these properties need to be maintained throughout the excecution, some elements may be removed from $O$ in every given iteration. The set $O_u$ is simply the set of elements removed from $O$ in the iteration in which $u$ is considered. Note that $O_u$ and $S_u$ are random sets, and Algorithm~\ref{alg:equiv_sample_greedy} defines values for them if it considers $u$ at some point. In the analysis below we assume that both $S_u$ and $O_u$ are empty if $u$ is not considered by Algorithm~\ref{alg:equiv_sample_greedy} (which means that Algorithm~\ref{alg:equiv_sample_greedy} does not explicitly set values for them).

We now explain how the algorithm manages to maintain the above mentioned properties of $O$. Let us begin with properties \ref{itm:O_ind} and \ref{itm:S_sub_O}. Clearly the removal of elements from $O$ that do not belong to $S$ cannot violate these properties, thus, we only need to consider the case that the algorithm adds an element $u$ to $S$ in some iteration. To maintain \ref{itm:S_sub_O}, $u$ is also added to $O$. On the one hand, $O + u$ might not be independent, and thus, the addition of $u$ to $O$ might violate \ref{itm:O_ind}. However, since $O$ is an extension of $S$ by \ref{itm:S_sub_O}, $S + u$ is independent by the choice of $u$ and $(\N, \mathcal{I})$ is a $k$-extendible system, the algorithm is able to choose a set $O_u \subseteq O \setminus S$ of size at most $k$ whose removal restores the independence of $O + u$, and thus, also \ref{itm:O_ind}. It remains to see why the algorithm preserves also \ref{itm:not_considered}. Since the algorithm never removes from $O$ an element of $S$, a violation of \ref{itm:not_considered} can only occur when the algorithm considers some element of $O \setminus S$. However, following this consideration one of two things must happen. Either $u$ is also added to $S$, or $u$ is placed in $O_u$ and removed from $O$. In either case \ref{itm:not_considered} is restored.

From this point on, every expression involving $S$ or $O$ is assumed to refer to the final values of these sets. The following lemma provides a lower bound on $f(S)$ that holds deterministically. Intuitively, this lemma follows from the observation that, when an element $u$ is considered by Algorithm~\ref{alg:equiv_sample_greedy}, its marginal contribution is at least as large as the marginal contribution of any element of $\OPT \setminus S$. \inConference{ Due to space constraints, many proofs in this section have been deferred to Appendix~\ref{app:missing_proofs_randomized}.}

\newcommand{\lowerBoundSLemma}{
$f(S) \geq f(S \cup \OPT) - \sum \limits_{u \in \N} |O_u \setminus S | \Delta f(u | S_u)$.
}
\begin{lemma} \label{lem:lower_bound_S}
\lowerBoundSLemma
\end{lemma}
\newcommand{\lowerBoundSProof}{
\begin{proof}
We first show that $f(S) \geq f(O)$, then we lower bound $f(O)$ to complete the proof. By \ref{itm:O_ind} and \ref{itm:S_sub_O}, we have $O \in \mathcal{I}$ and $S \subseteq O$, and thus, $S + v \in \mathcal{I}$ for all $v \in O \setminus S$ because $(\N, \mathcal{I})$ is an independence system. Consequently, the termination condition of Algorithm~\ref{alg:equiv_sample_greedy} guarantees that $\Delta f(v | S) \leq 0$ for all $v \in O \setminus S$. To use these observations, let us denote the elements of $O \setminus S$ by $v_1, v_2, \dotsc, v_{|O \setminus S|}$ in an arbitrary order. Then
\[
f(O) = f(S) + \sum \limits_{i=1}^{ |O \setminus S|} \Delta f \left( v_i | S \cup \{v_1, \dots , v_{|O \setminus S|} \} \right) \leq f(S) + \sum \limits_{i=1}^{ |O \setminus S|} \Delta f \left( v_i | S \right) \leq f(S)\enspace,
\]
where the first inequality follows by the submodularity of $f$.

It remains to prove the lower bound on $f(O)$. By definition, $O$ is the set obtained from $\OPT$ after the elements of $\cup_{u \in \N} O_u$ are removed and the elements of $S$ are added. Additionally, an element that is removed from $O$ is never added to $O$ again, unless it becomes a part of $S$. This implies that the sets $\left\{ O_u \setminus S \right\}_{u \in \N}$ are disjoint and that $O$ can also be written as
\begin{equation} \label{eq:rewriting_O}
O = \left( S \cup \OPT \right) \setminus \cup_{u \in \N} \left(O_u \setminus S \right)
\enspace.
\end{equation}
Denoting the the elements of $\N$ by $u_1,u_2, \dotsc, u_n$ in an arbitrary order, and using the above, we get
\begin{align*}
f(O)
	&= f(S \cup \OPT) - \sum \limits_{i=1}^n \Delta f \left( O_{u_i} \setminus S | (S \cup \OPT) \setminus \cup_{1 \leq j < i} ( O_{u_i} \setminus S) \right) &\text{(Equality~\eqref{eq:rewriting_O})} \\
	&\geq f(S \cup \OPT) - \sum \limits_{i=1}^n \Delta f(O_{u_i} \setminus S | S_{u_i}) \\
	&\geq f(S \cup \OPT) - \sum \limits_{i=1}^n \sum \limits_{v \in O_{u_i} \setminus S} \Delta f(v | S_{u_i}) \\
	&= f(S \cup \OPT) - \sum \limits_{u \in \cN} \sum \limits_{v \in O_u \setminus S} \Delta f(v | S_{u_i})
	\enspace,
\end{align*}
where the first inequality follows from the submodularity of $f$ because $S_{u_i} \subseteq S \subseteq (S \cup \OPT) \setminus \cup_{u \in \N} \left(O_u \setminus S \right)$ and the second inequality follows from the submodularity of $f$ as well.

To complete the proof of the lemma we need one more observation. Consider an element $u$ for which $O_u$ is not empty. Since $O_u$ is not empty, we know that $u$ was considered by the algorithm at some iteration. Moreover, every element of $O_u$ was also a possible candidate for consideration at this iteration, and thus, it must be the case that $u$ was selected for consideration because its marginal contribution with respect to $S_u$ is at least as large as the marginal contribution of every element of $O_u$. Plugging this observation into the last inequality, we get the following desired lower bound on $f(O)$.
\begin{align*}
	f(O)
	\geq{} &
	f(S \cup \OPT) - \sum \limits_{u \in \cN} \sum \limits_{v \in O_u \setminus S} \Delta f(v | S_{u_i}) \inConference{\\}
	\geq\inConference{{} &}
	f(S \cup \OPT) - \sum \limits_{u \in \cN} \sum \limits_{v \in O_u \setminus S} \Delta f(u | S_{u_i}) \inArXiv{\\}
	=\inArXiv{{} &}
	f(S \cup \OPT) - \sum \limits_{u \in \cN} |O_u \setminus S| \Delta f(u | S_{u_i})
	\enspace. \inArXiv{\qedhere}
\end{align*}
\end{proof}\vspace{0pt}
}\inArXiv{\lowerBoundSProof}

While the previous lemma was true deterministically, the next two lemmas are statements about expected values. At this point, it is convenient to define a new random variable. For every element $u \in \N$, let $X_u$ be an indicator for the event that $u$ is considered by Algorithm~\ref{alg:equiv_sample_greedy} in one of its iterations. The next lemma gives an expression for the expected value of $S$.\inConference{ It follows quite easily from the linearity of expectation.}

\newcommand{\expSLemma}{
$\E{f(S)} \geq \frac{1}{k+1} \sum \limits_{u \in \N} \E{X_u \Delta f(u | S_u)}$.
}
\begin{lemma} \label{lem:exp_S}
\expSLemma
\end{lemma}
\newcommand{\expSProof}{
\begin{proof}
For each $u \in \N$, let $G_u$ be a random variable whose value is equal to in the increase in the value of $S$ when $u$ is added to $S$ by Algorithm~\ref{alg:equiv_sample_greedy}. If $u$ is never added to $S$ by Algorithm~\ref{alg:equiv_sample_greedy}, then the value of $G_u$ is simply $0$. Clearly,
\[f(S) = f(\varnothing) + \sum \limits_{u \in \N} G_u \geq \sum \limits_{u \in \N} G_u \enspace.\]
By the linearity of expectation, it only remains to show that
\begin{equation} \label{eq:exp_G}
\E{G_u} = \frac{1}{k+1} \E{X_u \Delta f(u | S_u)} \enspace.
\end{equation}

Let $\mathcal{E}_u$ be an arbitrary event specifying all random decisions made by Algorithm~\ref{alg:equiv_sample_greedy} up until the iteration in which it considers $u$ if $u$ is considered, or all random decisions made by Algorithm~\ref{alg:equiv_sample_greedy} throughout its execution if $u$ is never considered. By the law of total probability, since these events are disjoint, it is enough to prove that Equality~(\ref{eq:exp_G}) holds when conditioned on every such event $\mathcal{E}_u$. If $\mathcal{E}_u$ is an event that implies that Algorithm~\ref{alg:equiv_sample_greedy} does not consider $u$, then, by conditioning on $\mathcal{E}_u$, we obtain
\[\E{G_u | \mathcal{E}_u}  = 0 = \frac{1}{k+1} \E{0 \cdot \Delta f(u|S_u) | \mathcal{E}_u} = \frac{1}{k+1} \E{X_u \Delta f(u | S_u) | \mathcal{E}_u} \enspace.\]
On the other hand, if $\mathcal{E}_u$ implies that Algorithm~\ref{alg:equiv_sample_greedy} does consider $u$, then we observe that $S_u$ is a deterministic set given $\cE_u$. Denoting this set by $S'_u$, we obtain
\[\E{G_u | \mathcal{E}_u} = \Pr\left[ u \in S \mid \cE_u\right] \Delta f(u | S'_u) = \frac{1}{k+1} \Delta f(u | S'_u) = \frac{1}{k+1} \E{X_u \Delta f(u | S_u) | \mathcal{E}_u} \enspace.\]
where the second equality hold since an element considered by Algorithm~\ref{alg:equiv_sample_greedy} is added to $S$ with probability $(k+1)^{-1}$.
\end{proof}\vspace{0pt}
}\inArXiv{\expSProof}

The next lemma relates terms appearing in the last two lemmata. Intuitively, this lemma shows that $O_u$ is on average a small set.\inConference{ For elements of $\OPT$ this is true since their $O_u$ set never contains more than one element, and for other elements this is true since their $O_u$ set is empty whenever they are not added to $S$.}

\newcommand{\linkLemma}[1]{
For every element $u \in \N$,
\begin{\ifx&#1&equation\else math\fi} \ifx&#1&\label{eq:linking_inequality}\fi
\E{ |O_u \setminus S | \Delta f(u | S_u) } \leq \frac{k}{k+1} \E{X_u \Delta f(u | S_u)}\enspace.
\end{\ifx&#1&equation\else math\fi}
}
\begin{lemma} \label{lem:link}
\inConference{\linkLemma{*}}\inArXiv{\linkLemma{}}
\end{lemma}
\newcommand{\linkProof}{
\begin{proof}
As in the proof of Lemma~\ref{lem:exp_S}, let $\mathcal{E}_u$ be an arbitrary event specifying all random decisions made by Algorithm~\ref{alg:equiv_sample_greedy} up until the iteration in which it considers $u$ if $u$ is considered, or all random decisions made by Algorithm~\ref{alg:equiv_sample_greedy} throughout its execution if it never considers $u$. By the law of total probability, since these events are disjoint, it is enough to prove Inequality~(\ref{eq:linking_inequality}) conditioned on every such event $\mathcal{E}_u$. If $\mathcal{E}_u$ implies that $u$ is not considered, then both $|O_u|$ and $X_u$ are $0$ conditioned on $\mathcal{E}_u$, and thus, the inequality holds as an equality. Thus, we may assume in the rest of the proof that $\mathcal{E}_u$ implies that $u$ is considered by Algorithm~\ref{alg:equiv_sample_greedy}. Notice that conditioned on $\cE_u$ the set $S_u$ is deterministic and $X_u$ takes the value $1$. Denoting the deterministic value of $S_u$ conditioned on $\cE_u$ by $S'_u$, Inequality~(\ref{eq:linking_inequality}) reduces to 
\[\E{|O_u \setminus S | \mid \mathcal{E}_u} \Delta f(u | S'_u) \leq \frac{k}{k+1} \Delta f(u | S'_u) \enspace. \]
Since $u$ is being considered, it must hold that $\Delta f(u |S'_u) > 0$, and thus, the above inequality is equivalent to $\E{|O_u \setminus S | \mid \mathcal{E}_u} \leq \frac{k}{k+1} $. There are now two cases to consider. If $\mathcal{E}_u$ implies that $u \in O$ at the beginning of the iteration in which Algorithm~\ref{alg:equiv_sample_greedy} considers $u$, then $O_u$ is empty if $u$ is added to $S$ and is $\{u\}$ if $u$ is not added to $S$. As $u$ is added to $S$ with probability $\frac{1}{k+1}$, this gives
\[\E{|O_u \setminus S | \mid \mathcal{E}_u} \leq \frac{1}{k+1} \cdot | \varnothing | + \left( 1- \frac{1}{k+1}\right) |\{u\}| = \frac{k}{k+1} \enspace, \]
and we are done. Consider now the case that $\mathcal{E}_u$ implies that $u \not \in O$ at the beginning of the iteration in which Algorithm~\ref{alg:equiv_sample_greedy} considers $u$. In this case, $O_u$ is always of size at most $k$ by the discussion following Algorithm~\ref{alg:equiv_sample_greedy}, and it is empty when $u$ is not added to $S$. As $u$ is, again, added to $S$ with probability $\frac{1}{k+1}$, we get in this case
\[ \E{|O_u \setminus S | \mid \mathcal{E}_u} \leq \frac{1}{k+1} \cdot k + \left( 1- \frac{1}{k+1}\right) |\varnothing| = \frac{k}{k+1} \enspace. \inArXiv{\qedhere} \]
\end{proof}\vspace{0pt}
}\inArXiv{\linkProof}

We are now ready to prove Theorem~\ref{thm:randomized_alg}.

\newcommand{\knownResultsProof}{
let us define $g:2^\N \rightarrow \nnR$ as $g(T) = f(T \cup \OPT)$ for every $T \subseteq \cN$. One can observe that $g$ is non-negative and submodular. Since $S$ contains every element with probability at most $(k + 1)^{-1}$, we get by Claim~\ref{lem:buchbinder}
\[
	\E{f(S \cup \OPT)}
	=
	\E{g(S)}
	\geq
	\left(1 - \frac{1}{k + 1}\right) g(\varnothing)
	=
	\frac{k}{k +1} f(\OPT)
	\enspace.
\]
}
\begin{proof}[Proof of Theorem~\ref{thm:randomized_alg}]
We prove here that Algorithm~\ref{alg:sample_greedy} achieves the approximation ratios guaranteed by Theorem~\ref{thm:randomized_alg} for submodular and monotone submodular objectives. The approximation ratio guaranteed by Theorem~\ref{thm:randomized_alg} for linear objectives is obtained, using similar ideas, by a close variant of Algorithm~\ref{alg:sample_greedy}; and we defer a more detailed discussion of this variant and its guarantee to Appendix~\ref{app:linear_objectives_algorithm}.

As discussed earlier, Algorithms~\ref{alg:sample_greedy} and~\ref{alg:equiv_sample_greedy} have identical output distributions, and so it suffices to show that Algorithm~\ref{alg:equiv_sample_greedy} achieves the desired approximation ratios. Note that
\begin{align*}
\E{f(S)} &\geq \E{f(S \cup \OPT)} - \sum \limits_{u \in \N} \E{|O_u \setminus S| \Delta f(u | S_u)} &\text{(Lemma~\ref{lem:lower_bound_S})} \\
&\geq \E{f(S \cup \OPT)} - \frac{k}{k+1} \sum \limits_{u \in \N} \E{X_u \Delta f(u | S_u)} &\text{(Lemma~\ref{lem:link})} \\
&\geq \E{f(S \cup \OPT)} - k \E{f(S)}\enspace. &\text{(Lemma~\ref{lem:exp_S})}
\end{align*}

If $f$ is monotone, then the proof completes by rearranging the above inequality and observing that $f(S \cup OPT) \geq f(OPT)$. Otherwise, \inArXiv{\knownResultsProof}\inConference{the fact that $S$ contains every element with probability at most $(k + 1)^{-1}$ can be used to show via known results (see Appendix~\ref{app:missing_proofs_randomized} for details) that
\[
	\E{f(S \cup \OPT)}
	\geq
	\frac{k}{k +1} f(\OPT)
	\enspace.
\]}
The proof now completes by combining the two above inequalities, and rearranging.
\end{proof}\vspace{0pt}

\section{Experimental Results} \label{sec:experimental_results}

In this section, we present results of an experiment on real dataset comparing the performance of \algdt and \algrd with other competitive algorithms. In particular, we compare our algorithms to the greedy algorithm and FANTOM, an $O(krn)$-time algorithm for nonmonotone submodular optimization introduced in \citep{Mirzasoleiman2016}. We also boost \algrd by taking the best of four runs, denoted Max Sample Greedy. We test these algorithms on a personalized movie recommendation system, and find that while  \algdt and \algrd return comparable solutions to FANTOM, they run orders of magnitude faster. These initial results indicate that \algrd may be applied (without any loss in performance) to massive problem instances that were previously intractable.

In the movie recommendation system application, we observe movie ratings from users, and our objective is to recommend movies to users based on their reported favorite genres. In particular, given a user-specified input of favorite genres, we would like to recommend a short list of movies that are diverse, and yet representative, of those genres. The similarity score between movies that we use is derived from user ratings, as in \citep{Lindgren2015}.

Let us now describe the problem setting in more detail. Let $\N$ be a set of movies, and $G$ be the set of all movie genres. For a movie $i \in \N$, we denote the set of genres of $i$ by $G(i)$ (each movie may have multiple genres). Similarly, for a genre $g \in G$, denote the set of all movies in that genre by $\N(g)$. Let $s_{i,j}$ be a non-negative similarity score between movies $i,j \in \N$, and suppose a user $u$ seeks a representative set of movies from genres $G_u \subseteq G$. Note that the set of movies from these genres is $\N_u = \cup_{g \in G_u} \N(g)$. Thus, a reasonable utility function for choosing a diverse yet representative set of movies $S$ for $u$ is 
\begin{equation} \label{eq:cut_fun}
f_u(S) = \sum \limits_{i \in S} \sum \limits_{j \in \N_u} s_{i,j} - \lambda \sum \limits_{i \in S} \sum \limits_{j \in S} s_{i,j}
\end{equation} 
for some parameter $0 \leq \lambda \leq 1$. Observe that the first term is a sum-coverage function that captures the representativeness of $S$, and the second term is a dispersion function penalizing similarity within $S$. Moreover, for $\lambda = 1$ this utility function reduces to the simple cut function. 

The user may specify an upper limit $m$ on the number of movies in his recommended set. In addition, he is also allowed to specify an upper limit $m_g$ on the number of movies from genre $g$ in the set for each $g \in G_u$ (we call the parameter $m_g$ a \emph{genre limit}). The first constraint corresponds to an $m$-uniform matroid over $\N_u$, while the second constraint corresponds to the intersection of $|G_u|$ partition matroids (each imposing the genre limit of one genre). Thus, our constraints in this movie recommendation corresponds to the intersection of $1 + |G_u|$ matroids, which is a $(1 + |G_u|)$-extendible system (in fact, a more careful analysis shows that it corresponds to a $|G_u|$-extendible system).

For our experiments, we use the MovieLens 20M dataset, which features 20 million ratings of 27,000 movies by 138,000 users. To obtain a similarity score between movies, we take an approach developed in \citep{Lindgren2015}. First, we fill missing entries of an incomplete movie-user matrix $M \in \mathbb{R}^{n \times m}$ via low-rank matrix completion \citep{Candes2008, Hastie2015}, then we randomly sample to obtain a matrix $\tilde{M} \in \mathbb{R}^{n \times k}$ where $k \ll m$ and the inner products between rows is preserved. The similarity score between movies $i$ and $j$ is then defined as the inner product of their corresponding rows in $\tilde{M}$. In our experiment, we set the total recommended movies limit to $m=10$. The genre limits $m_g$ are always equal for all genres, and we vary them from 1 to 9. Finally, we set our favorite genres $G_u$ as Adventure, Animation and Fantasy. For each algorithm and test instance, we record the function value of the returned solution $S$ and the number of calls to $f$, which is a machine-independent measure of run-time.

\begin{figure}[h] 
\centering  
\subfigure[Solution Quality]{\label{fig:fun_val}\includegraphics[width=60mm]{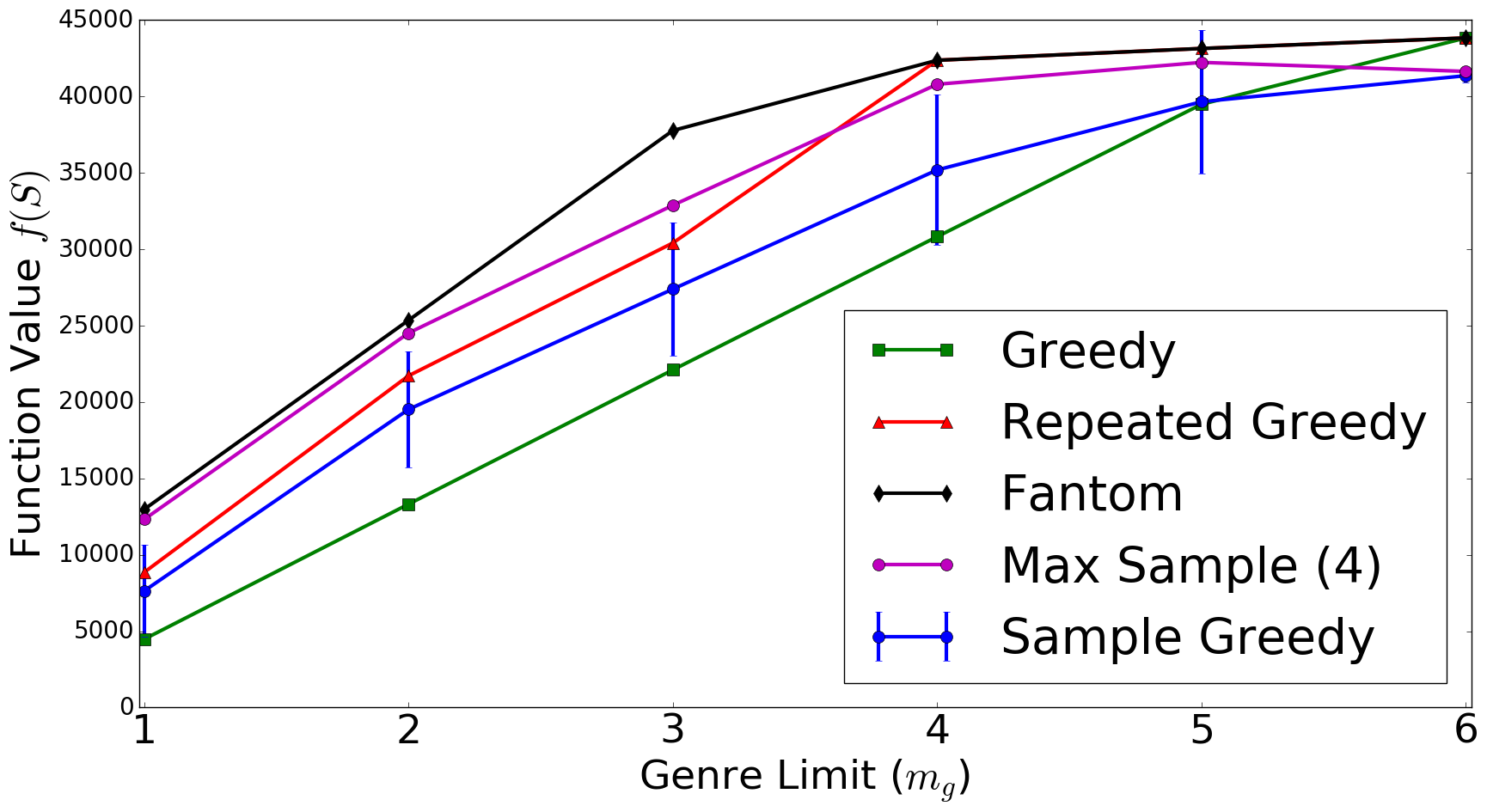}}
\subfigure[Run Time]{\label{fig:run_time}\includegraphics[width=60mm]{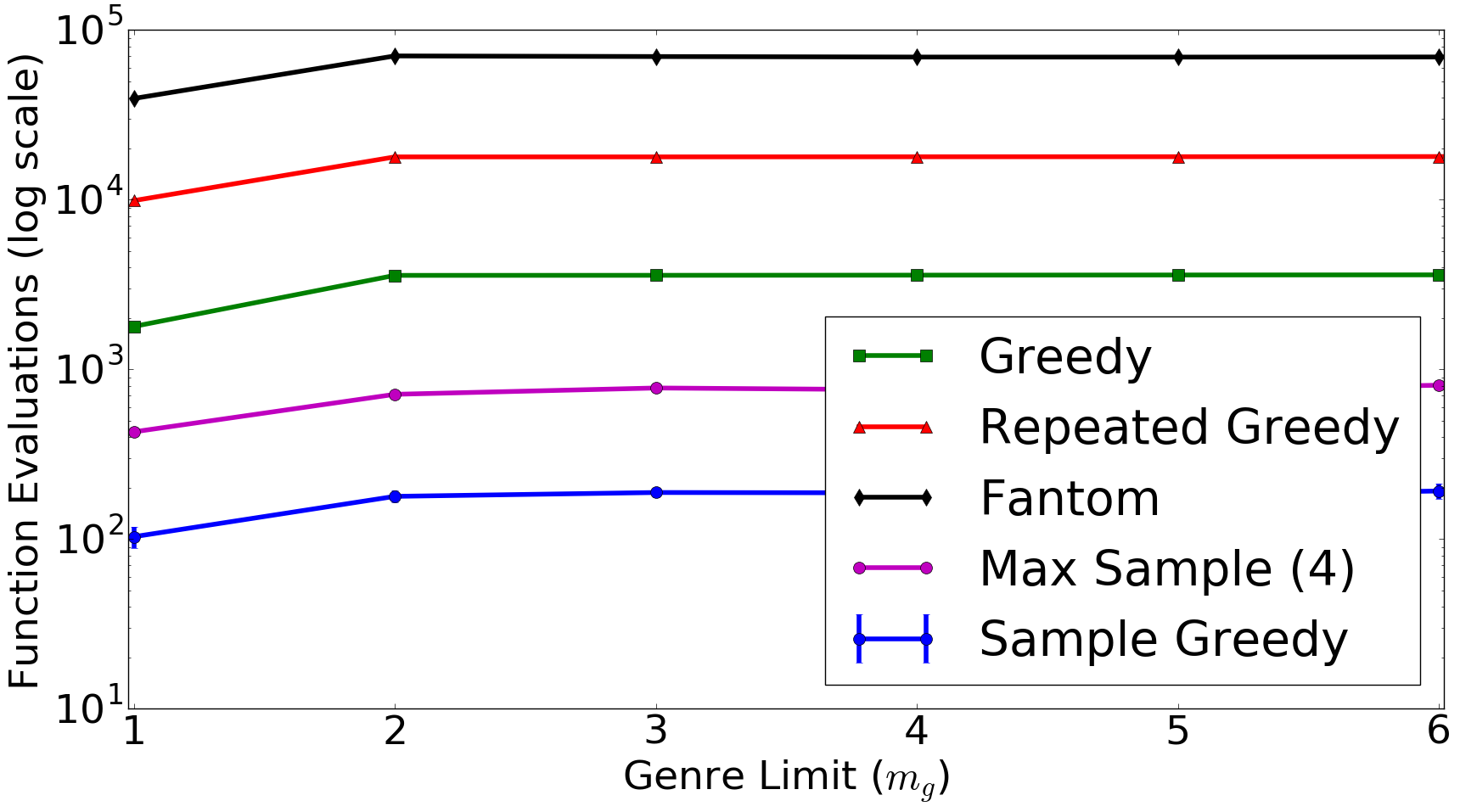}} \\
\subfigure[Ratio Comparison $m_g = 1$]{\label{fig:ratio_comp_1}\includegraphics[width=60mm]{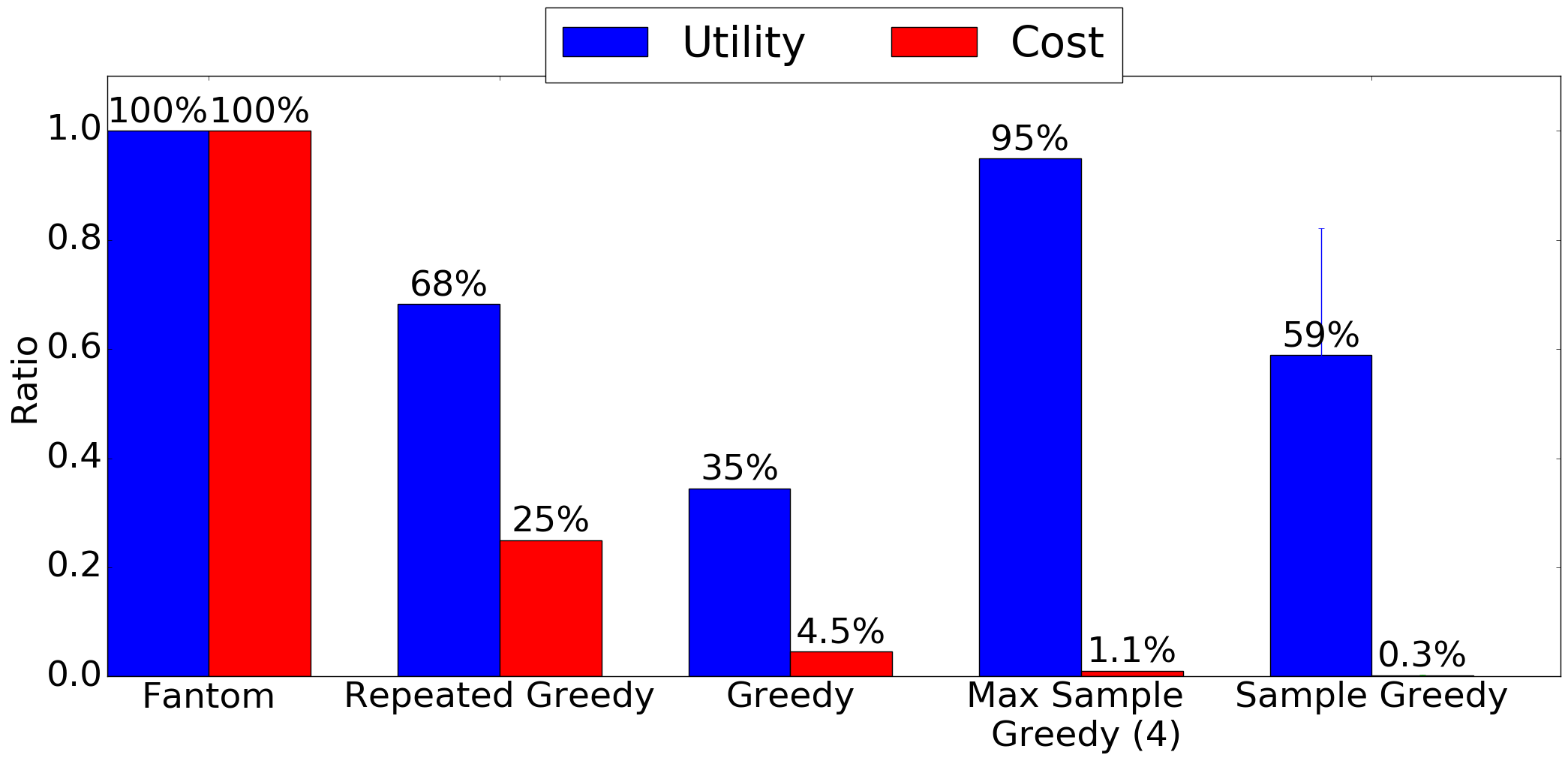}}
\subfigure[Ratio Comparison $m_g = 4$]{\label{fig:ratio_comp_4}\includegraphics[width=60mm]{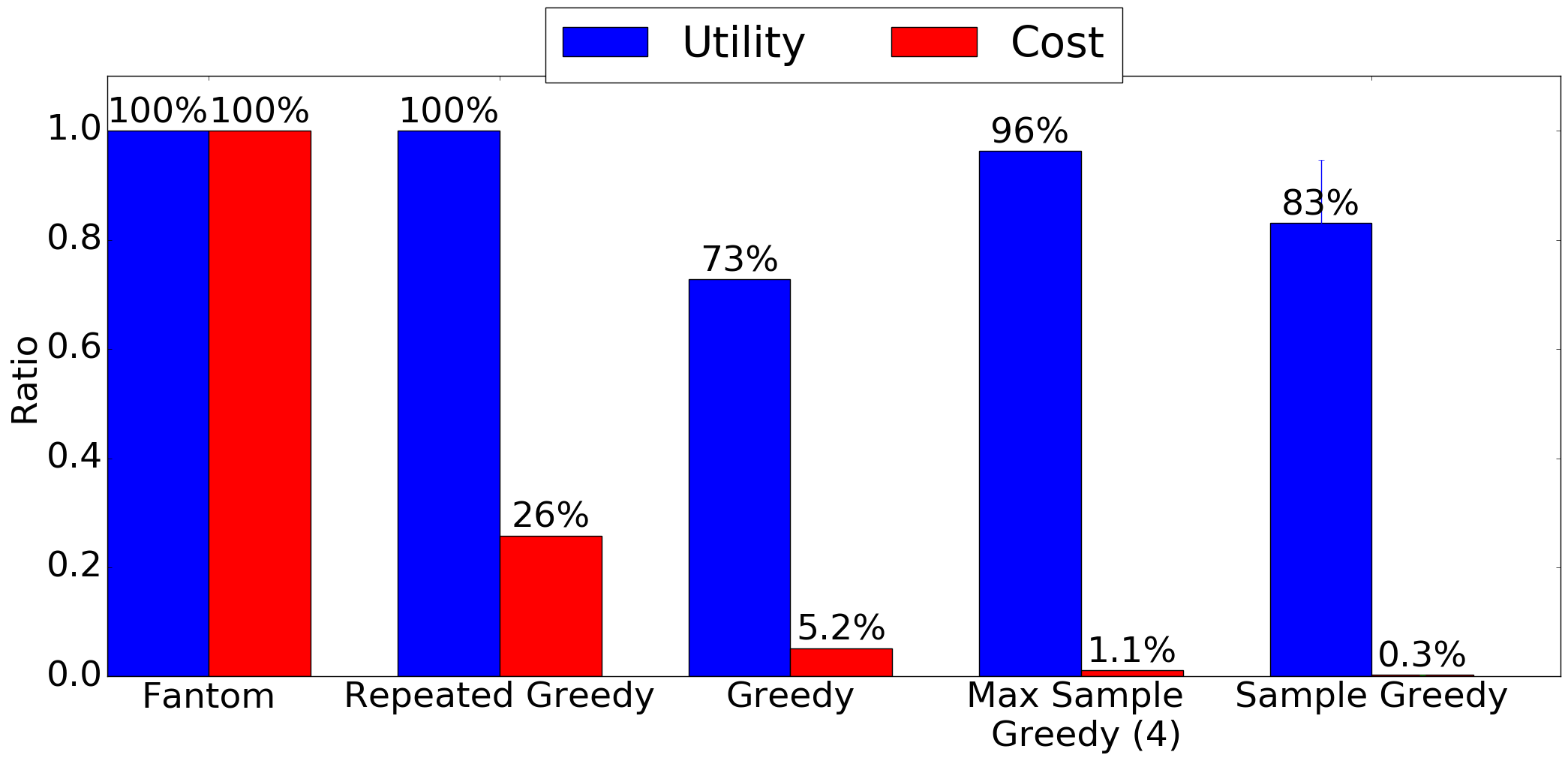}} 
\small \\
\caption{Performance Comparison. \ref{fig:fun_val} shows the function value of the returned solutions for tested algorithms with varying genre limit $m_g$. \ref{fig:run_time} shows the number of function evaluations on a logarithmic scale with varying genre limit $m_g$. \ref{fig:ratio_comp_1} and \ref{fig:ratio_comp_4} show the ratio of solution quality and cost with FANTOM as a baseline.} \label{fig:performance}
\end{figure}

Figure~\ref{fig:fun_val} shows the value of the solution sets for the various algorithms. As we see from Figure~\ref{fig:fun_val}, FANTOM consistently returns a solution set with the highest function value; however, \algdt and \algrd return solution sets with similarly high function values. We see that Max Sample Greedy, even for four runs, significantly increases the performance for more constrained problems. Note that for such more constrained problems, Greedy returns solution sets with much lower function values. Figure~\ref{fig:run_time} shows the number of function calls made by each algorithm as the genre limit $m_g$ is varied. For each algorithm, the number of function calls remains roughly constant as $m_g$ is varied---this is due to the lazy greedy implementation that takes advantage of submodularity to reduce the number of function calls. We see that for our problem instances, \algdt runs about an order of magnitude faster than FANTOM and \algrd runs roughly three orders of magnitude faster than FANTOM. Moreover, boosting \algrd by executing it a few times does not incur a significant increase in cost.

To better analyze the tradeoff between the utility of the solution value  and the cost of run time, we compare the ratio of these measurements for the various algorithms using FANTOM as a baseline. See Figure~\ref{fig:ratio_comp_1} and \ref{fig:ratio_comp_4} for these ratio comparisons for genre limits $m_g = 1$ and $m_g = 4$, respectively. For the case of $m_g=1$, we see that boosted \algrd provides nearly the same utility as FANTOM, while only incurring 1.09\% of the computational cost. Likewise, for the case of $m_g =4$, \algdt achieves the same utility as FANTOM, while incurring only a quarter of the cost. Thus, we may conclude that our algorithms provide solutions whose quality is on par with current state of the art, and yet they run in a small fraction of the time.

\begin{figure}[h] 
\centering  
\includegraphics[width=\textwidth]{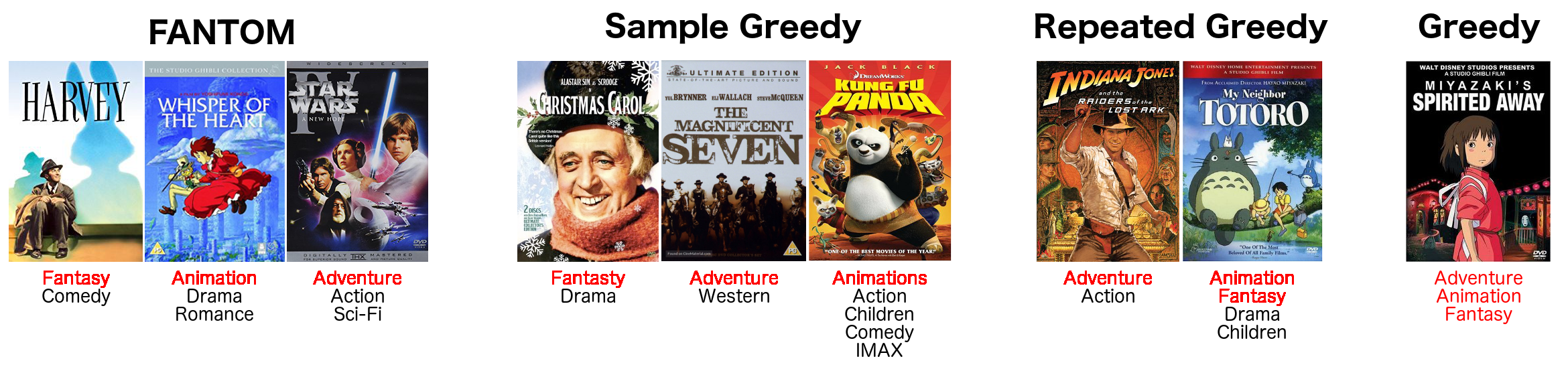} 
\caption{Solution Sets. The movies in the solution sets for $m_g=1$ returned by FANTOM, Sample Greedy, Repeated Greedy and Greedy are listed here, along with genre information. The favorite genres ($G_u$) are in red.} \label{fig:solution_sets}
\end{figure}

While Greedy may get stuck in poor locally optimal solutions, \algdt and \algrd avoid this by greedily combing through the solution space many times and selecting random sets, respectively. Fortunately, the movie recommendation system has a very interpretable solution so we can observe this phenomenon. See Figure~\ref{fig:solution_sets} for the movies recommended by the different algorithms. Because $m_g = 1$, we are constrained here to have at most one movie from Adventure, Animation and Fantasy. As seen in Figure~\ref{fig:solution_sets}, FANTOM and \algrd return maximum size solution sets that are both diverse and representative of these genres. On the other hand, Greedy gets stuck choosing a single movie that belongs to all three genres, thus, precluding any other choice of movie from the solution set. 

%
%% Acknowledgments---Will not appear in anonymized version
%\acks{This material is based upon work supported by the National Science Foundation Graduate Research Fellowship under Grant No. (NSF grant number). The authors would like to thank Eric Lindgren for kindly sharing processed movieLens data.}
%
\bibliographystyle{plainnat}
\bibliography{./tex/references-sub}
%
%\externaldocument{preliminaries}

\appendix

\section{Hardness of Maximization over \texorpdfstring{$k$}{k}-Extendible Systems} \label{sec:hardness}
In this appendix, we prove Theorems~\ref{thm:linear_hardness} and~\ref{thm:submodular_hardness}. The proof consists of two steps. In the first step, we will define two $k$-extendible systems which are indistinguishable in polynomial time. The inapproximability result for linear objectives will follow from the indistinguishability of these systems and the fact that the size of their maximal sets are very different. In the second step we will define monotone submodular objective functions for the two $k$-extendible systems. Using the symmetry gap technique of~\citep{Vondrak2013}, we will show that these objective functions are also indistinguishable, despite being different. Then, we will use the differences between the objective functions to prove the slightly stronger inapproximability result for monotone submodular objectives.

Given three positive integers $k, h$ and $m$ such that $h$ is an integer multiple of $2k$, let us construct a $k$-extendible system $\MM(k, h, m) = (\NN_{k. h, m}, \II_{k, h, m})$ as follows. The ground set of the system is $\NN_{k, h, m} = \cup_{i = 1}^h H_i(k,m)$, where $H_i(k, m) = \{u_{i, j} ~|~ 1 \leq j \leq km\}$. A set $S \subseteq \NN_{k, h, m}$ is independent (i.e., belongs to $\II_{k, h, m}$) if and only if it obeys the following inequality:
\[
	g_{m,k}(|S \cap H_1(k, m)|) + |S \setminus H_1(k, m)|
	\leq
	m
	\enspace,
\]
where the function $g_{m,k}$ is defined by
\[
	g_{m,k}(x) = \min\left\{x, \frac{2km}{h}\right\} + \max\left\{\frac{x - 2km/h}{k}, 0\right\}
	\enspace.
\]
Intuitively, a set is independent if its elements do not take too many ``resources'', where most elements requires a unit of resources, but elements of $H_1(k, m)$ take only $1/k$ unit of resources each once there are enough of them. Consequently, the only way to get a large independent set is to pack many $H_1(k, m)$ elements.

\begin{lemma}
For every choice of $h$ and $m$, $\MM(k, h, m)$ is a $k$-extendible system.
\end{lemma}
\begin{proof}
First, observe that $g(x)$ is a monotone function, and therefore, a subset of an independent set of $\MM(k, h, m)$ is also independent. Also, $g(0) = 0$, and therefore, $\varnothing \in \II_{k, h, m}$. This proves that $\MM(k, h, m)$ is an independence system. In the rest of the proof we show that it is also $k$-extendible.

Consider an arbitrary independent set $C \in \II_{k, h, m}$, an independent extension $D$ of $C$ and an element $u \not \in D$ for which $C + u \in \II_{k, h, m}$. We need to find a subset $Y \subseteq D \setminus C$ of size at most $k$ such that $D \setminus Y + u \in \II_{k, h, m}$. If $|D \setminus C| \leq k$, then we can simply pick $Y = D \setminus C$. Thus, we can assume from now on that $|D \setminus C| > k$.

Let $\Sigma(S) = g(|S \cap H_1(k, m)|) + |S \setminus H_1(k, m)|$. By definition, $\Sigma(D) \leq m$ because $D \in \II_{k, h, m}$. Observe that $g(x)$ has the property that for every $x \geq 0$, $k^{-1} \leq g(x + 1) - g(x) \leq 1$. Thus, $\Sigma(S)$ increases by at most $1$ every time that we add an element to $S$, but decreases by at least $1/k$ every time that we remove an element from $S$. Hence, if we let $Y$ be an arbitrary subset of $D \setminus C$ of size $k$, then
\[
	\Sigma(D \setminus Y + u)
	\leq
	\Sigma(D) - \frac{|Y|}{k} + 1
	=
	\Sigma(D)
	\leq
	m
	\enspace,
\]
which implies that $D \setminus Y + u \in \II_{k, h, m}$.
\end{proof}\vspace{0pt}

%For proving hardness for maximizing a linear function over a $k$-Extendible System, we need a series of $k$-Extendible Systems. Fix two constants $k$ and $h \geq 2k$. Our series of $k$-Extensible Systems is constructed by choosing different values for $m$. Notice that the size of $\NN_{k,h,m}$ is $khm$, which is $O(m)$ since $k$ and $h$ are constant. We consider the problem of maximizing the cardinality function subject to the above series of $k$-Extendible Systems.

Before presenting the second $k$-extendible system, let us show that $\MM(k, h, m)$ contains a large independent set.

\begin{observation} \label{ob:large_independent_set}
$\MM(k, h, m)$ contains an independent set whose size is $k(m - 2km/h) + 2km/h \geq mk(1-2k/h)$. Moreover, there is such set in which all elements belong to $H_1(k, m)$.
\end{observation}
\begin{proof}
Let $s = k(m - 2km/h) + 2km/h$, and consider the set $S = \{u_{1, j} \mid 1 \leq j \leq s\}$. This is a subset of $H_1(k, m) \subseteq \NN_{k, h, m}$ since $s \leq km$. Also,
\begin{align*}
	g(|S|)
	={} &
	g(s) 
	=
	\min\left\{s, \frac{2km}{h}\right\} + \max\left\{\frac{s - 2km/h}{k}, 0\right\}\\
	\leq{} &
	\frac{2km}{h} + \max\left\{\frac{[k(m - 2km/h) + 2km/h] - 2km/h}{k}, 0\right\}\\
	={} &
	\frac{2km}{h} + \max\left\{m - \frac{2km}{h}, 0\right\}
	=
	m
	\enspace.
\end{align*}
Since $S$ contains only elements of $H_1(k, m)$, its independence follows from the above inequality.
\end{proof}\vspace{0pt}

Let us now define our second $k$-extendible system $\MM'(k, h, m) = (\NN_{k, h, n}, \II'_{k, n})$. The ground set of this system is the same as the ground set of $\MM(k, h, m)$, but a set $S \subseteq \NN_{k, h, m}$ is considered independent in this independence system if and only if its size is at most $m$. Clearly, this is a $k$-extendible system (in fact, it is a uniform matroid). Moreover, note that the ratio between the sizes of the maximal sets in $\cM(k, h, m)$ and $\cM'(k, h, m)$ is at least
\[
	\frac{mk(1 - 2k/h)}{m}
	=
	k(1 - 2k/h)
	\enspace.
\]
Our plan is to show that it takes exponential time to distinguish between the systems $\MM(k, h, m)$ and $\MM'(k, h, m)$, and thus, no polynomial time algorithm can provide an approximation ratio better than this ratio for the problem of maximizing the cardinality function (i.e., the function $f(S) = |S|$) subject to a $k$-extendible system constraint.

Consider a polynomial time deterministic algorithm that gets either $\MM_{k, h, m}$ or $\MM'_{k, h, m}$ after a random permutation was applied to the ground set. We will prove that with high probability the algorithm fails to distinguish between the two possible inputs. Notice that by Yao's lemma, this implies that for every random algorithm there exists a permutation for which the algorithms fails with high probability to distinguish between the inputs.

Assuming our deterministic algorithm gets $\MM'_{k, h, m}$, it checks the independence of a polynomial collection of sets. Observe that the sets in this collection do not depend on the permutation because the independence of a set in $\MM'_{k, h, m}$ depends only on its size, and thus, the algorithm will take the same execution path given every permutation. If the same algorithm now gets $\MM_{k, h, m}$ instead, it will start checking the independence of the same sets until it will either get a different answer for one of the checks (different than what is expected for $\MM'_{k, h, m}$) or it will finish all the checks. Note that in the later case the algorithm must return the same answer that it would have returned had it been given $\MM'_{k, h, m}$. Thus, it is enough to upper bound the probability that any given check made by the algorithm will result in a different answer given the inputs $\MM_{k, h, m}$ and $\MM'_{k, h, m}$.

\begin{lemma} \label{le:independence_different_unlikely}
Following the application of the random ground set permutation, the probability that a set $S$ is independent in $\MM_{k, h, m}$ but not in $\MM'_{k, h, m}$, or vice versa, is at most $e^{-\frac{2km}{h^2}}$.
\end{lemma}
\begin{proof}
Observe that as long as we consider a single set, applying the permutation to the ground set is equivalent to replacing $S$ with a random set of the same size. So, we are interested in the independence in $\MM_{k, h, m}$ and $\MM'_{k, h, m}$ of a random set of size $|S|$. If $|S| > km$, then the set is never independent in either $\MM_{k, h, m}$ or $\MM'_{k, h, m}$, and if $|S| \leq m$, then the set is always independent in both $\MM_{k, h, m}$ and $\MM'_{k, h, m}$. Thus, the interesting case is when $m < |S| \leq km$.

Let $X = |S \cap H_1(k, m)|$. Notice that $X$ has a hypergeometric distribution, and $\mathbb{E}[X] = |S|/h$. Thus, using bounds given in \citep{Skala13} (these bounds are based on results of \citep{Chvatal79,Hoeffding1963}), we get
\[
	\Pr\left[X \geq \frac{2km}{h}\right]
	=
	\Pr\left[X \geq \bE[|X|] + \frac{km}{h}\right]
	\leq
	e^{-2\left(\frac{km/h}{|S|}\right)^2 \cdot |S|}
	=
	e^{-\frac{2k^2m^2}{h^2 \cdot |S|}}
	\leq
	e^{-\frac{2km}{h^2}}
	\enspace.
\]
The lemma now follows by observing that $X \leq 2km/h$ implies that $S$ is a dependent set under both $\MM_{k, h, m}$ and $\MM'_{k, h, m}$.
\end{proof}\vspace{0pt}

We now think of $m$ as going to infinity and of $h$ and $k$ as constants. Notice that given this point of view the size of the ground set $\cN_{k, h, m}$ is $nkh = O(m)$. Thus, the last lemma implies, via the union bound, that with high probability an algorithm making a polynomial number (in the size of the ground set) of independence checks will not be able to distinguishes between the cases in which it gets as input $\cM_{k, h, m}$ or $\cM'_{k, h, m}$.

We are now ready to prove Theorem~\ref{thm:linear_hardness}.

\begin{proof}[Proof of Theorem~\ref{thm:linear_hardness}]
Consider an algorithm that needs to maximize the cardinality function over the $k$-extendible system $\cM_{k, h, m}$ after the random permutation was applied, and let $T$ be its output set. Notice that $T$ must be independent in $\cM_{k, h, m}$, and thus, its size is always upper bounded by $mk$. Moreover, since the algorithm fails, with high probability, to distinguish between $\cM_{k, h, m}$ and $\cM'_{k, h, m}$, $T$ is with high probability also independent in $\cM'_{k, h, m}$, and thus, has a size of at most $m$. Therefore, the expected size of $T$ cannot be larger than $m + o(1)$.

On the other hand, Lemma~\ref{ob:large_independent_set} shows that $\cM_{k, h, m}$ contains an independent set of size at least $mk(1 - 2k/h)$. Thus, the approximation ratio of the algorithm is no better than
\[
	\frac{mk(1 - 2k/h)}{m + o(1)}
	\geq
	\frac{mk(1 - 2k/h)}{m} - \frac{k}{m} o(1)
	=
	k - 2k^2/h - o(1)
	\enspace.
\]
Choosing a large enough $h$ (compared to $k$), we can make this approximation ratio larger than $k - \ee$ for any constant $\ee > 0$.
\end{proof} \vspace{0in}

To prove a stronger inapproximability result for monotone submodular objectives, we need to associate a monotone submodular function with each one of our $k$-extendible systems. Towards this goal, consider the monotone submodular function $f_h : 2^{\NN_h} \rightarrow \mathbb{R}^+$ defined over the ground set $\NN_h = [h]$ by
\[
	f_h(S)
	=
	\min\{|S|, 1\}
	\enspace.
\]

Let $F_h\colon [0, 1]^{\cN_h} \to \nnR$ be the mutlilinear extension of $f_h$, i.e., $F_h(x) = \bE[f_h(R(x))]$ for every vector $x \in [0, 1]^{\cN_h}$ (where $R(x)$ is a random set containing every element $u \in \cN_u$ with probability $x_u$, independently). Additionally, given a vector $x \in [0, 1]^{\NN_h}$, let us define $\bar{x} = (\| x \|_1 / h) \cdot \characteristic_{\NN_h}$. Notice that $f_h$ is invariant under any permutation of the elements of $\NN_h$. Thus, by Lemma~3.2 in \citep{Vondrak2013}, for every $\ee' > 0$ there exists $\delta_h > 0$ and two functions $\hat{F}_h, \hat{G}_h : [0,1]^{\NN_h} \rightarrow \nnR$ with the following properties.
\begin{itemize}
	\item For all $x \in [0, 1]^{\NN_h}$: $\hat{G}_h(x) = \hat{F}_h(\bar{x})$.
	\item For all $x \in [0, 1]^{\NN_h}$, $|\hat{F}_h(x) - F_h(x)| \leq \ee'$.
	\item Whenever $|x - \bar{x}|^2_2 \leq \delta_h$, $\hat{F}_h(x) = \hat{G}_h(x)$.
	\item The first partial derivatives of $\hat{F}_h$ and $\hat{G}_h$ are absolutely continuous.
	\item $\frac{\partial \hat{F}_h}{\partial x_u}, \frac{\partial \hat{G}_h}{\partial x_u} \geq 0$ everywhere for every $u \in \NN_h$.
	\item $\frac{\partial^2 \hat{F}_h}{\partial x_u \partial x_v}, \frac{\partial^2 \hat{G}_h}{\partial x_u \partial x_v} \leq 0$ 
		almost everywhere for every pair $u,v \in \NN_h$.
\end{itemize}

The objective function we associate with $\MM(k, h, m)$ is $\hat{F}_{h}(y(S))$, where $y(S)$ is a vector in $[0, 1]^{\NN_h}$ whose $i^{th}$ coordinate is $|S \cap H_i(k, m)| / (km)$. Similarly, the objective function we associate with $\MM'(k, h, m)$ is $\hat{G}_{h}(y(S))$. Notice that both objective functions are monotone and submodular by Lemma~3.1 of \citep{Vondrak2013}. %From now on we assume $h$ was selected as a large enough constant to guarantee $h \geq 8k/\ee$.
We now bound the maximum value of a set in $\MM(k, h, m)$ and $\MM'(k, h, m)$ with respect to their corresponding objective functions.

\begin{lemma} \label{le:submodular_good_opt}
The maximum value of a set in $\MM(k, h, m)$ with respect to the objective $\hat{F}_{h}(y(S))$ is at least $1 - 2k/h - \ee'$.
\end{lemma}
\begin{proof}
Observation~\ref{ob:large_independent_set} guarantees the existence of an independent set $S \subseteq H_1(k, m)$ of size $s \geq k(m - 2km/h)$ in $\MM(k, h, m)$. The objective value associated with this set is
\[
	\hat{F}_h(y(S))
	\geq
	F_h(y(S)) - \ee'
	=
	\frac{s}{km} - \ee'
	\geq
	\frac{k(m - 2km/h)}{km} - \ee'\\
	=
	1 - 2k/h - \ee'
	\enspace.
\]
\end{proof}\vspace{0pt}

\begin{lemma} \label{le:submodular_bad_opt}
The maximum value of a set in $\MM'(k, h, m)$ with respect to the objective $\hat{G}_{h}(y(S))$ is at most $1 - e^{-1/k} + h^{-1} + \ee'$.
\end{lemma}
\begin{proof}
The objective $\hat{G}_{h}(y(S))$ is monotone. Thus, the maximum value set in $\MM'(k, h, m)$ must be of size $m$. Notice that for every set $S$ of this size, we get
\begin{align*}
	\hat{G}_h(y(S))
	={} &
	\hat{F}_h(\overline{y(S)})
	=
	\hat{F}_h((kh)^{-1} \cdot \characteristic_{\NN_h})
	\leq
	F_h((kh)^{-1} \cdot \characteristic_{\NN_h}) + \ee'\\
	={} &
	1 - \left(1 - \frac{1}{kh}\right)^h + \ee'
	\leq
	1 - e^{-1/k}\left(1 - \frac{1}{k^2h}\right) + \ee'
	\leq
	1 - e^{-1/k} + h^{-1} + \ee'
	\enspace.
\end{align*}
\end{proof}\vspace{0pt}

%\begin{corollary} \label{co:submodular_ratio}
%The ratio between the values of the optimal sets in $\MM(k, h m)$ and $\MM'(k, h, m)$ is at most: $1 - e^{-1/k} + \ee$.
%\end{corollary}
%By Lemmata~\ref{le:submodular_bad_opt} and \ref{le:submodular_good_opt}, the ratio that we need to bound is at most:
%\begin{proof}
%\[
	%\frac{1 - e^{-1/k} + \ee/4}{1 - 3\ee/8}
	%\leq
	%[1 - e^{-1/k} + \ee/4] + 3\ee / 8
	%=
	%1 - e^{-1/k} + 5\ee/8
	%<
	%1 - e^{-1/k} + \ee
	%\enspace,
%\]
%where the first inequality holds for $\ee < 8e^{-1/k}/3$.
%\end{proof}\vspace{0pt}

As before, our plan is to show that after a random permutation is applied to the ground set it is difficult to distinguish between $\MM(k, h, m)$ and $\MM'(k, h, m)$ even when each one of them is accompanied with its associated objective. This will give us an inapproximability result which is roughly equal to the ratio between the bounds given by the last two lemmata.

Observe that Lemma~\ref{le:independence_different_unlikely} holds regardless of the objective function. Thus, $\cM(k, h, m)$ and $\cM'(k, h, m)$ are still polynomially indistinguishable. Additionally, the next lemma shows that their associated objective functions are also polynomially indistinguishable.

\begin{lemma} \label{le:value_different_unlikely}
Following the application of the random ground set permutation, the probability that any given set $S$ gets two different values under the two possible objective functions is at most $2h \cdot e^{-2mk\delta_h /h^2}$.
\end{lemma}
\begin{proof}
Recall that, as long as we consider a single set $S$, applying the permutation to the ground set is equivalent to replacing $S$ with a random set of the same size. Hence, we are interested in the value under the two objective functions of a random set of size $|S|$. Define $X_i = |S \cap H_i(k, m)|$. Since $X_i$ has the a hypergeometric distribution, the bound of \citep{Skala13} gives us
\begin{align*}
	\Pr\left[X_i \geq \frac{|S|}{h} + mk \cdot \sqrt{\frac{\delta_h}{h}} \right]
	={} &
	\Pr\left[X_i \geq \mathbb{E}[X_i] + mk \cdot \sqrt{\frac{\delta_h}{h}} \right]\\
	\leq{} &
	e^{-2 \cdot \left(\frac{mk \cdot \sqrt{\delta_h/h}}{|S|}\right)^2 \cdot |S|}
	=
	e^{-\frac{2\delta_h}{h} \cdot \frac{m^2k^2}{|S|}}
	\leq
	e^{-2mk\delta_h/h^2}
	\enspace.
\end{align*}
Similarly, we also get
\[
	\Pr\left[X_i \leq \frac{|S|}{h} - mk \cdot \sqrt{\frac{\delta_h}{h}} \right]
	\leq
	e^{-2mk\delta_h /h^2}
	\enspace.
\]
Combining both inequalities using the union bound now yields
\[
	\Pr\left[\left|X_i - \frac{|S|}{h}\right| \geq mk \cdot \sqrt{\frac{\delta_h}{h}} \right]
	\leq
	2 e^{-2mk\delta_h /h^2}
	\enspace.
\]

Using the union bound again, the probability that $\left|X_i - \frac{|S|}{h}\right| \geq mk \cdot \sqrt{\frac{\delta_h}{h}}$ for any $1 \leq i \leq h$ is at most $2h \cdot e^{-2mk\delta_h /h^2}$. Thus, to prove the lemma it only remains to show that the value of the two objective functions for $S$ are equal when $\left|X_i - \frac{|S|}{h}\right| < mk \cdot \sqrt{\frac{\delta_h}{h}}$ for every $1 \leq i \leq h$.

Notice that $\overline{y(S)}$ is a vector in which all the coordinates are equal to $|S|/ (mkh)$. Thus, the inequality $\left|X_i - \frac{|S|}{h}\right| < mk \cdot \sqrt{\frac{\delta_h}{h}}$ is equivalent to $[y_i(S) - \overline{y_i(S)}] < \sqrt{\frac{\delta_h}{h}}$. Hence,
\[
	|y(S) - \overline{y(S)}|_2^2
	=
	\sum_{i = 1}^h (y_i(S) - \overline{y_i(S)})^2
	<
	\sum_{i = 1}^h \left(\sqrt{\frac{\delta_h}{h}}\right)^2
	=
	\sum_{i = 1}^h \frac{\delta_h}{h}
	=
	\delta_h
	\enspace,
\]
which implies the lemma by the properties of $\hat{F}_h$ and $\hat{G}_h$.
\end{proof}\vspace{0pt}
 
Consider a polynomial time deterministic algorithm that gets either $\MM_{k, h, m}$ with its corresponding objective or $\MM'_{k, h, m}$ with its corresponding objective after a random permutation was applied to the ground set. Consider first the case that the algorithm gets $\MM'_{k, h, m}$ (and its corresponding objective). In this case, the algorithm checks the independence and value of a polynomial collection of sets (we may assume, without loss of generality, that the algorithm checks both things for every set that it checks). As before, one can observe that the sets in this collection do not depend on the permutation because the independence of a set in $\MM'_{k, h, m}$ and its value with respect to $\hat{G}_h(y(S)) = \hat{F}_h(\overline{y(S)})$ depend only on the set's size, which guarantees that the algorithm takes the same execution path given every permutation. If the same algorithm now gets $\MM_{k, h, m}$ instead, it will start checking the independence and values of the same sets until it will either get a different answer for one of the checks (different than what is expected for $\MM'_{k, h, m}$) or it will finish all the checks. Note that in the later case the algorithm must return the same answer that it would have returned had it been given $\MM'_{k, h, m}$.

By the union bound, Lemmata~\ref{le:independence_different_unlikely} and~\ref{le:value_different_unlikely} imply that the probability that any of the sets whose value or independence is checked by the algorithm will result in a different answer for the two inputs decreases exponentially in $m$, and thus, with high probability the algorithm fails to distinguish between the inputs, and returns the same output for both. Moreover, note that by Yao's principal this observation extends also to polynomial time randomized algorithms. 

We are now ready to prove Theorem~\ref{thm:submodular_hardness}.

\begin{proof}[Proof of Theorem~\ref{thm:submodular_hardness}]
Consider an algorithm that whose objective is to maximize $\hat{F}(y(S))$ over the $k$-extendible system $\cM_{k, h, m}$ after the random permutation was applied, and let $T$ be its output set. Notice that $\hat{F}(y(T)) \leq \hat{F}(\characteristic_{\cN_h}) \leq F(\characteristic_{\cN_h}) + \ee' = 1 + \ee'$. Moreover, the algorithm fails, with high probability, to distinguish between $\cM_{k, h, m}$ and $\cM'_{k, h, m}$. Thus, with high probability $T$ is independent in $\cM'(k, h, m)$ and has the same value under both objective functions $\hat{F}(y(S))$ and $\hat{G}(y(S))$, which implies, by Lemma~\ref{le:submodular_bad_opt}, $\hat{F}(y(S)) = \hat{G}(y(S)) \leq 1 - e^{-1/k} + h^{-1} + \ee'$. Hence, in conclusion we proved
\[
	\bE[\hat{F}(y(T))]
	\leq
	1 - e^{-1/k} + h^{-1} + \ee' + o(1)
	\enspace.
\]

On the other hand, Lemma~\ref{le:submodular_good_opt} shows that $\cM_{k, h, m}$ contains an independent set of value of at least $1 - 2k/h - \ee'$ (with respect to $\hat{F}(y(S))$). Thus, the approximation ratio of the algorithm is no better than
\begin{align*}
	&\frac{1 - 2k/h - \ee'}{1 - e^{-1/k} + h^{-1} + \ee' + o(1)}\\
	\geq{} &
	(1 - e^{-1/k} + h^{-1} + \ee' + o(1))^{-1} - (1 - e^{-1/k})^{-1}(2k/h + \ee')\\
	\geq{} &
	(1 - e^{-1/k})^{-1} - (1 - e^{-1/k})^{-2}(h^{-1} + \ee' + o(1)) - (1 - e^{-1/k})^{-1}(2k/h + \ee')\\
	\geq{} &
	(1 - e^{-1/k})^{-1} - (k + 1)^2(h^{-1} + \ee' + o(1)) - (k + 1)(2k/h + \ee')
	\enspace,
\end{align*}
where the last inequality holds since $1 - e^{-1/k} \geq (k + 1)^{-1}$.
Choosing a large enough $h$ (compared to $k$) and a small enough $\ee'$ (again, compared to $k$), we can make this approximation ratio larger than $(1 - e^{-1/k})^{-1} - \ee$ for any constant $\ee > 0$.
\end{proof} \vspace{0in}

\section{Improved Approximation Guarantee for Linear Functions} \label{app:linear_objectives_algorithm}

In Section~\ref{sec:randomized_alg} we described Algorithm~\ref{alg:sample_greedy}, which obtains a $(k + 1)$-approximation for the maximization of a monotone submodular function subject to a $k$-extendible system constraint. One can observe that the approximation ratio of this algorithm is no better than $k + 1$ even when the function is linear since it takes no element with probability larger than $(k + 1)^{-1}$. In this appendix we show that if the algorithm is allowed to select each element with probability $k^{-1}$, then its approximation ratio for linear functions improves to $k$; which proves the guarantee of Theorem~\ref{thm:randomized_alg} for linear objectives.

Specifically, we consider in this appendix a variant of Algorithm~\ref{alg:sample_greedy} in which the sampling probability in Line~\ref{line:sample} is changed to $k^{-1}$. Similarly, the probability in Line~\ref{line:sample_equiv} of Algorithm~\ref{alg:equiv_sample_greedy} is also changed to $k^{-1}$ in order to keep the two algorithms equivalent. In the rest of this appendix, any reference to Algorithms~\ref{alg:sample_greedy} and~\ref{alg:equiv_sample_greedy} should be implicitly understood as a reference to the variants of these algorithms with the above changes.

The analysis we present here for Algorithm~\ref{alg:equiv_sample_greedy} (and thus, also for its equivalent Algorithm~\ref{alg:sample_greedy}) is very similar to the one given in Section~\ref{sec:randomized_alg}, and it is based on the same ideas. However, slightly more care is necessary in order to establish the improved approximation ratio of $k$. We begin with the following lemma, which corresponds to Lemma~\ref{lem:lower_bound_S} from Section~\ref{sec:randomized_alg}. For every $u \in \cN$, let $Y_u$ be a random variable which takes the value $1$ if $u \in S$ and, in addition, $u$ does not belong to $O$ at the beginning of the iteration in which $u$ is considered. In every other case the value of $Y_u$ is $0$.

\begin{lemma} \label{lem:linear_lower_bound_S}
$f(S) \geq f(\OPT) - \sum \limits_{u \in \N} [|O_u| - Y_u] f(u)$.
\end{lemma}
\begin{proof}
The proof of Lemma~\ref{lem:lower_bound_S} begins by showing that $f(S) \geq f(O)$. This part of the proof is still true since it is independent of the sampling probability. Thus, we only need to show that
\[
	f(O)
	\geq
	f(\OPT) - \sum \limits_{u \in \N} [|O_u| - Y_u] f(u)
	\enspace.
\]
Recall that $O$ beings as equal to $\OPT$. Thus, to prove the last inequality it is enough to show that the second term on its right hand side is an upper bound on the decrease in the value of $O$ over time. In the rest of the proof we do this by showing that $[|O_u| - Y_u] f(u)$ is an upper bound on the decrease in the value of $O$ in the iteration in which $u$ is considered, and is equal to $0$ when $u$ is not considered at all. 

Let us first consider the case that $u$ is not considered at all. In this case, by definition, $O_u = \varnothing$ and $Y_u = 0$, which imply together $[|O_u| - Y_u] f(u) = 0 \cdot f(u) = 0$. Consider now the case that $u$ is considered by Algorithm~\ref{alg:equiv_sample_greedy}. In this case $O$ is changed during the iteration in which $u$ is considered in two ways. First, the elements of $O_u$ are removed from $O$, and second, $u$ is added to $O$ if it is added to $S$ and it does not already belong to $O$. Thus, the decrease in the value of $O$ during this iteration can be written as
\[
	\sum_{v \in O_u} f(v) - Y_u \cdot f(u)
	\enspace.
\]
To see why this expression is lower bounded by $[|O_u| - Y_u] f(u)$, we recall that in the proof of Lemma~\ref{lem:lower_bound_S} we showed that $f(v|S_u) \leq f(u|S_u)$ for every $v \in O_u$, which implies, since $f$ is linear, $f(v) \leq f(u)$ for every such element $v$.
\end{proof}\vspace{0pt}

The next lemma corresponds to Lemma~\ref{lem:exp_S} from Section~\ref{sec:randomized_alg}. We omit its proof since it is completely identical to the proof of Lemma~\ref{lem:exp_S} up to change in the sampling probability.

\begin{lemma} \label{lem:linear_exp_S}
$\E{f(S)} \geq \frac{1}{k} \sum \limits_{u \in \N} \E{X_u \Delta f(u|S_u)} = \frac{1}{k} \sum \limits_{u \in \N} \E{X_u} f(u)$.
\end{lemma}

We need one last lemma which corresponds to Lemma~\ref{lem:link} from Section~\ref{sec:randomized_alg}. The role of this lemma is to relate terms appearing in the last two lemmata.

\begin{lemma} \label{lem:linear_link}
For every element $u \in \N$,
\begin{equation} \label{eq:linear_linking_inequality}
\E{|O_u| - Y_u} \leq \frac{k - 1}{k} \E{X_u}\enspace.
\end{equation}
\end{lemma}
\begin{proof}
As in the proofs of Lemma~\ref{lem:link}, let $\mathcal{E}_u$ be an arbitrary event specifying all random decisions made by Algorithm~\ref{alg:equiv_sample_greedy} up until the iteration in which it considers $u$ if $u$ is considered, or all random decisions made by Algorithm~\ref{alg:equiv_sample_greedy} throughout its execution if it never considers $u$. By the law of total probability, since these events are disjoint, it is enough to prove Inequality~(\ref{eq:linear_linking_inequality}) conditioned on every such event $\mathcal{E}_u$. If $\mathcal{E}_u$ implies that $u$ is not considered, then $|O_u|$, $X_u$ and $Y_u$ are all $0$ conditioned on $\mathcal{E}_u$, and thus, the inequality holds as an equality. Thus, we may assume in the rest of the proof that $\mathcal{E}_u$ implies that $u$ is considered by Algorithm~\ref{alg:equiv_sample_greedy}. Notice that, conditioned on $\cE_u$, $X_u$ takes the value $1$. Hence, Inequality~(\ref{eq:linear_linking_inequality}) reduces to 
\[\E{|O_u| - Y_u\mid \mathcal{E}_u} \leq \frac{k - 1}{k} \enspace. \]
There are now two cases to consider. The first case is that $\mathcal{E}_u$ implies that $u \in O$ at the beginning of the iteration in which Algorithm~\ref{alg:equiv_sample_greedy} considers $u$. In this case $Y_u = 0$, and in addition, $O_u$ is empty if $u$ is added to $S$, and is $\{u\}$ if $u$ is not added to $S$. As $u$ is added to $S$ with probability $\frac{1}{k}$, this gives
\[\E{|O_u| - Y_u \mid \mathcal{E}_u} \leq \frac{1}{k} \cdot | \varnothing | + \left( 1- \frac{1}{k}\right) |\{u\}| = \frac{k-1}{k} \enspace, \]
and we are done. Consider now the case that $\mathcal{E}_u$ implies that $u \not \in O$ at the beginning of the iteration in which Algorithm~\ref{alg:equiv_sample_greedy} considers $u$. In this case, if $u$ is not added to $S$, then we get $Y_u = 0$ and $O_u = \varnothing$. In contrast, if $u$ is added to $S$, then $Y_u = 1$ by definition and $|O_u| \leq k$ by the discussion following Algorithm~\ref{alg:equiv_sample_greedy}. As $u$ is, again, added to $S$ with probability $\frac{1}{k}$, we get in this case
\[ \E{|O_u| - Y_u | \mid \mathcal{E}_u} \leq \frac{1}{k} \cdot (k - 1) + \left( 1- \frac{1}{k}\right) |\varnothing| = \frac{k - 1}{k} \enspace. \inArXiv{\qedhere} \]
\end{proof}\vspace{0pt}

We are now ready to prove the guarantee of Theorem~\ref{thm:randomized_alg} for linear objectives.
\begin{proof}[Proof of Theorem~\ref{thm:randomized_alg} for linear objectives.]
We prove here that the approximation ratio guaranteed by Theorem~\ref{thm:randomized_alg} for linear objectives is obtained by (the modified) Algorithm~\ref{alg:sample_greedy}. As discussed earlier, Algorithms~\ref{alg:sample_greedy} and~\ref{alg:equiv_sample_greedy} have identical output distributions, and so it suffices to show that Algorithm~\ref{alg:equiv_sample_greedy} achieves this approximation ratio. Note that
\begin{align*}
\E{f(S)} &\geq f(\OPT) - \sum \limits_{u \in \N} \bE[|O_u| - Y_u] f(u) &\text{(Lemma~\ref{lem:linear_lower_bound_S})} \\
&\geq f(\OPT) - \frac{k-1}{k}\sum \limits_{u \in \N} \bE[X_u] f(u) &\text{(Lemma~\ref{lem:linear_link})} \\
&\geq f(\OPT) - (k - 1) \E{f(S)}\enspace. &\text{(Lemma~\ref{lem:linear_exp_S})}
\end{align*}
The proof now completes by rearranging the above inequality.
\end{proof}\vspace{0pt}

%\section{Proofs of Propositions} \label{sec:appendix_basic_results}
%
%In this section, we provide proofs of the more basic Propositions provided in the main paper. We start by giving a proof of Proposition~\ref{prop:submodular_fact1}, which is a basic property of submodular functions.
%
%\begin{proof}[Proof of Proposition~\ref{prop:submodular_fact1}]
%The proof is easy and it goes here.
%\end{proof}\vspace{0pt}

%Finally, we prove the bound on $1 - e^{-\frac{1}{k - \epsilon}} $ which was used in Section~\ref{subsec:main_contributions}. This inequality demonstrates that our approximation ratios are very close to our derived hardness bounds.
%
%\begin{proposition} \label{claim:e_bounds}
%For $k \geq 1$, $$\frac{1}{k+1} \leq 1 - e^{-\frac{1}{k}} \leq \frac{1}{k+ \frac{1}{2}}$$
%\end{proposition}
%\begin{proof}
%The proof is easy and it goes here.
%\end{proof}\vspace{0pt}

\inConference{

\section{Missing Proofs of Section~\ref{sec:deterministic_alg}} \label{app:missing_proofs_deterministic}

\begin{replemma}{lem:known_approx_results}
\knownApproxResultsLemma
\end{replemma}
\knownApproxResultsProof

\begin{replemma}{lem:sum_greedy_results}
\sumGreedyResultsLemma
\end{replemma}
\sumGreedyResultsProof

\begin{replemma}{lem:submodular_fact1}
\submodularFactLemma
\end{replemma}
\submodularFactProof

In the proof of Theorem~\ref{thm:deterministic_alg} we omitted the calculation showing that for $\ell = \lceil \sqrt{k} \rceil$, we get
\calculations

\section{Missing Proofs of Section~\ref{sec:randomized_alg}} \label{app:missing_proofs_randomized}

\begin{replemma}{lem:lower_bound_S}
\lowerBoundSLemma
\end{replemma}
\lowerBoundSProof

\begin{replemma}{lem:exp_S}
\expSLemma
\end{replemma}
\expSProof

\begin{replemma}{lem:link}
\linkLemma{}
\end{replemma}
\linkProof

In the proof of Theorem~\ref{thm:randomized_alg} we omitted the explanation why the inequality $\E{f(S \cup \OPT)} \geq \frac{k}{k +1} f(\OPT)$ holds. To see why this is the case, \knownResultsProf

}

\end{document}